%% file: main.tex
\documentclass{article}
\usepackage{fullpage}

\usepackage[utf8]{inputenc} 
\usepackage[T1]{fontenc}    
\usepackage{hyperref}       
\usepackage{url}            
\usepackage{booktabs}       
\usepackage{amsfonts}       
\usepackage{nicefrac}       
\usepackage{microtype}      
\usepackage{lipsum}		
\usepackage{graphicx}
\usepackage{natbib}
\usepackage{doi}
\usepackage{amsthm}
\usepackage{color}
\usepackage{amsfonts}
\usepackage{amsmath}
\usepackage{algorithm}
\usepackage[noend]{algpseudocode}
\usepackage{xspace}
\usepackage{enumitem}
\usepackage{natbib}

\title{Federated Functional Gradient Boosting}

\date{} 					

\author{ 
\quad	\begin{tabular}{c}
	Zebang Shen\\
University of Pennsylvania\\
\texttt{zebang@seas.upenn.edu}
	\end{tabular}
	\and
\quad
	\begin{tabular}{c}
	Hamed Hassani \\
University of Pennsylvania\\
\texttt{hassani@seas.upenn.edu} 
	\end{tabular}
	\and
	\and
	\begin{tabular}{c}
	Satyen Kale \\
Google Research\\
\texttt{satyenkale@google.com}
	\end{tabular}
\quad
	\and
	\quad
	\begin{tabular}{c}
	Amin Karbasi \\
Yale \\
\texttt{amin.karbasi@yale.edu} 
	\end{tabular}
}

\hypersetup{
pdftitle={Federated Functional Gradient Boosting},
}

\input{required.tex}

\newcommand{\fedavg}{\textsc{FedAvg}\xspace}
\newcommand{\scaffold}{\textsc{SCAFFOLD}\xspace}

\newcommand{\ffgd}{\textsc{FFGB}\xspace}

\newcommand{\ffgdc}{\textsc{FFGB.C}\xspace}
\newcommand{\ffgdl}{\textsc{FFGB.L}\xspace}
\newcommand{\Proj}[1]{\text{Proj}_{#1}}
\newcommand{\Clip}[1]{\Pi_{#1}}
\newcommand{\TV}{\mathrm{TV}}
\newcommand{\lip}{\mathrm{lip}}

\begin{document}
\maketitle


\input{abstract.tex}

\input{introduction.tex}


\input{preliminaries.tex}

\input{federated_functional_minimization.tex}

\input{methodology.tex}



\input{Experiment.tex}

\input{practicalities}

\bibliographystyle{plainnat}
\bibliography{example_paper}  

\newpage
\appendix
\input{appendix.tex}

\end{document}

%% file: abstract.tex
\begin{abstract}
	In this paper, we initiate a study of functional minimization in Federated Learning. First, in the semi-heterogeneous setting, when the marginal distributions of the feature vectors on client machines are identical, we develop the federated functional gradient boosting (\ffgd) method that provably converges to the global minimum. Subsequently, we extend our results to the fully-heterogeneous setting (where marginal distributions of feature vectors may differ) by designing an efficient variant of \ffgd called \ffgdc, with provable convergence to a neighborhood of the global minimum within a radius that depends on the total variation distances between the client feature distributions. For the special case of square loss, but still in the fully heterogeneous setting, we design the \ffgdl method that also enjoys provable convergence to a neighborhood of the global minimum but within a radius depending on the much tighter Wasserstein-1 distances. For both \ffgdc and \ffgdl, the radii of convergence shrink to zero as the feature distributions become more homogeneous. Finally, we conduct proof-of-concept experiments to demonstrate the benefits of our approach against natural baselines. 
\end{abstract}

%% file: introduction.tex
\section{Introduction} \label{section_introduction}
Federated learning (FL) is a machine learning paradigm in which multiple clients cooperate to learn a model under the orchestration of a central server \citep{mcmahan2017communication}. In FL, clients tend to have very heterogeneous data, and this data is never sent to the central server due to privacy and communication constraints. Thus, it is necessary to offload as much computation as possible to the client devices. The challenge in FL is to train a single unified model via decentralized computation even though the clients have heterogeneous data.

Current research in FL mainly focuses on training parametric models where the variable to be optimized is typically a vector containing the parameters of a neural network. While convenient, the parametric model limits the feasible domain of the output function. Expert domain knowledge must be incorporated when designing the structure of the parametric model to ensure that a sufficiently good solution is feasible.

In this paper, we initiate a study of {\em non-paramteric functional minimization} methods for FL as a means of overcoming the aforementioned limitations of parametric models. To the best of our knowledge, this is the first work to study the use of functional minimization methods in the context of FL. In functional minimization, the variable to be trained is the target function itself. The feasible domain of such a non-parametric learning paradigm can be as large as the family of the square-integrable functions, which allows tremendous flexibility in representing the desired models.

To solve the functional minimization problem, a powerful tool is the {\em restricted functional gradient descent} method (RFGD), also known as functional gradient boosting. Methods in this class select and aggregate simple functions (or the so-called {\em weak learners}) in an iterative manner and have been proved to perform well in terms of minimizing the training loss as well as generalization.

\paragraph{Contributions.} This paper formulates the federated functional minimization problem and develops algorithms and convergence analyses for various settings of heterogeneity of clients. The main contributions  are as~follows:

\begin{enumerate}[wide, labelwidth=!, labelindent=0pt]
	\item \textbf{Algorithms for the semi-heterogeneous setting.} We first consider the semi-heterogeneous setting where the marginal distributions of the feature vectors on the clients are identical and heterogeneity arises due to differing conditional distributions of labels. In this setting, we propose the {\em federated functional gradient boosting} (\ffgd) method. The algorithm relies on the clients running multiple RFGD steps locally, but augmented with \emph{residual} variables that are crucial for proving global convergence of the algorithm.
	\item \textbf{Benefit of local steps in the semi-heterogeneous setting.} Our analysis of \ffgd suggests that the number of local steps of \ffgd should be $\Omega(\sqrt{T})$ in order to have the best convergence rate where $T$ is the number of global communication rounds. This is a rare result that shows the benefit of taking multiple local steps in FL.
	\item \textbf{Algorithms for the fully-heterogeneous setting.} Next, we consider the fully heterogeneous setting where the entire joint distribution of features and labels varies across clients. In this setting, we propose two different algorithms, \ffgdc and \ffgdl, which are variants of \ffgd. Our analysis shows that \ffgdc converges to a neighborhood of the global optimal solution, where the radius of the neighborhood depends on the heterogeneity of the data measured in terms of the {\em total variation} (TV) distance between the marginal distributions of the features of the clients. The algorithm \ffgdl operates under the square loss and our analysis shows a similar convergence result to that of \ffgdc, with the difference being that the degree of heterogeneity and radius of the neighborhood is measured in terms of the {\em Wasserstein-1} distance, which can be significantly smaller than the TV distance. Importantly, for both \ffgdc and \ffgdl, the radius of the convergence neighborhood gracefully \emph{diminishes to zero} when problem degenerates to the semi-heterogeneous case.
	\item \textbf{Experiments.} While the main focus of this paper is to develop the theoretical foundations of federated functional minimization, we also performed some basic proof-of-concept experiments to provide evidence for the benefits of the methods developed in this paper on practical datasets. Our experiments show superior performance of our methods compared to natural baselines.
\end{enumerate}

\subsection{Related work}

For the standard parametric optimization problem, \fedavg \citep{mcmahan2017communication} is the most popular algorithm in the federated setting. In every round of \fedavg, the server sends a global average parameter to the local machines.
The local machines then perform multiple local steps (usually stochastic gradient descent) to update the received parameters.
These improved local parameters are then aggregated by the server for the next round. It has been noted by several papers (see, e.g. \citep{karimireddy2020scaffold} and the references therein) that \fedavg deteriorates in the presence of client heterogeneity. Federated optimization algorithms such as \scaffold \citep{karimireddy2020scaffold}, \textsc{FedProx}~\citep{fedprox}, \textsc{Mime}~\citep{karimireddy2020mime}, etc. were designed to tackle this issue.

{\em Boosting} \citep{freund-schapire-book} is a classical ensemble method for building additive models in a greedy, stagewise manner. In each stage, a new classifier is added to the current ensemble to decrease the training loss. In the regression setting, this is typically done by finding an approximation to the functional gradient of the loss in the base model class via a training procedure (also referred to as a ``weak learning oracle'') for the base class (see, e.g., \citep{mason2000boosting,friedman}). Rigorous analysis with rates of convergence for procedures of this form were given by \citep{duffy-helmbold,ratsch-et-al,zhang-yu,grubb2011generalized}. The algorithms in this paper and their analysis build upon the prior work of \citep{grubb2011generalized}, with the main challenges coming from the restriction in the FL setting that no client data be shared with the server, as well as the heterogeneity in client data.

In the FL setting, stagewise boosting algorithms as described above have not been studied, to the best of our knowledge. An ensemble method in the FL setting called \textsc{FedBoost} was developed by \citet{hamer2020fedboost} to compute a convex combination of \emph{pre-trained} classifiers to minimize the training loss. The algorithm is quite different from standard classical boosting methods since it doesn't operate in a greedy stagewise manner, nor does it incorporate training of the base classifiers. Since this setting is quite different from the one considered in this paper, we do not provide further comparison with this method.

%% file: preliminaries.tex
\section{Preliminaries} \label{section_preliminary}
In this section, we define the necessary notation as well as the functional minimization problem.
\vspace{-2mm}
\input{notation}

\subsection{Functional Minimization in Weighted $\LM^2$ Space}
For some output space $\YM$, let $\ell:\RBB^c\times\YM\rightarrow\RBB$ be a loss function that is convex in the first argument. { An important example of $\ell$ is the cross entropy loss, where $\YM = [c]$, and
\begin{equation} \label{eqn_loss_cross_entropy}
	\ell(y', y) = -\log\left(\frac{\exp(y'_y)}{\sum_{i=1}^{c}\exp(y'_i)}\right).
\end{equation}}
Given a joint distribution $P$ on $\XM \times \YM$, we use $\alpha$ to denote the marginal distribution of $P$ on $\XM$, and for every $x\in\XM$, we let $\beta^x\in\MM_+^1(\YM)$ be the distribution on $\YM$ under $P$ conditioned on $x$. We then define the risk functional $\RM:\LM^2(\alpha)\rightarrow \RBB$ as
\begin{equation}\label{eqn_def_obj_func}
	\RM[f]\ \defi\ \EBB_{(x, y) \sim P}[\ell(f(x), y)].
\end{equation}
We consider the following Tikhonov regularized functional minimization problem:
\begin{equation} \label{eqn_functional_minimization}
	\min_{f\in\LM^2(\alpha)} \FM[f] \defi \RM[f] + \frac{\mu}{2}\|f\|^2_\alpha.
\end{equation}
Here, $\mu\geq0$ is some regularization parameter.
To solve such a problem, we define the \emph{subgradient} $\nabla \FM[f]$ of $\FM$ at $f$ as the set of all $g\in\LM^2(\alpha)$ such that
\begin{equation}
	\FM[g] \geq \FM[f] + \langle g - f, \nabla \FM[f]\rangle_{\alpha}.
\end{equation}
Since $\LM^2(\alpha)$ is a Hilbert space, according to the Riesz representation theorem, the subgradient can also be represented by a function in $\LM^2(\alpha)$.
Concretely, the subgradient can be explicitly computed as follows: $\nabla \FM[f]$ is the collection of functions $h\in\LM^2{(\alpha)}$, such that for any $x \in \supp(\alpha)$,
\begin{equation} \label{sub-comp}
	h(x) = \EBB_{y\sim\beta^x} [\nabla_1 \ell(f(x), y)] + \mu f(x),
\end{equation}
where $\nabla_1 \ell(y', y) = {\partial \ell(y',y)}/{\partial y'}$. In particular, when only empirical measures of $\alpha$ and $\beta^x$ are available, i.e. when $P$ is the empirical distribution on the set $\{(x_1, y_1), (x_2, y_2), \ldots, (x_M, y_M) \in \XM \times \YM\}$, with $x_i \neq x_j$ for $1 \leq i < j \leq M$, the empirical version of the above subgradient computation \eqref{sub-comp} is as follows: For all $j \in [M]$
\begin{equation} \label{eqn_subgradient_empirical}
	h(x_j) = \nabla_1 \ell(f(x_j), y_j) + \mu f(x_j).
\end{equation}

\paragraph{Restricted Functional Gradient Descent (Boosting).}
A standard approch to solve the functional minimization problem \eqref{eqn_functional_minimization} is the functional gradient descent method
\begin{equation}
	f^{t+1} := f^{t} - \eta^t h^t,\ h^t\in \nabla \RM[f^t].
\end{equation}
In the empirical case \eqref{eqn_subgradient_empirical}, one can construct $h^t \in \nabla \RM[f^t]$ by interpolating $\{x_j,  \nabla_1 \ell(f^t(x_j), y_j)\}_{j=1}^M$. However, this choice of $h^t$ has at least two drawbacks: 1. Evaluating $h^t$ at a single point requires going through the whole dataset; 
2. $h^t$ is constructed explicitly on the data points $(x_j, y_j)$ which provides no privacy protection.\\
One alternative to the explicit functional gradient descent method is the restricted functional gradient descent, also known as Boosting \citep{mason2000boosting}:
\begin{equation} \label{eqn_rfgd}
	f^{t+1} := f^{t} - \eta^t h_{weak}^t,\ h_{weak}^t = \mathrm{WO}^2_{\alpha} (\nabla \RM[f^t]).
\end{equation}
Here, $\mathrm{WO}^2_{\alpha}$ is a \emph{weak learning oracle} such that for any $\phi\in\LM^2(\alpha)$, the output $h = \mathrm{WO}^2_{\alpha}(\phi)$ is a function in $\LM^2(\alpha)$ satisfying the following \emph{weak learning assumption}:
\begin{equation} \label{eqn_weak_oracle_contraction}
	\|h -  \phi\|_{\alpha} \leq (1-\gamma)\|\phi\|_{\alpha}, 
\end{equation}
for some positive constant $0< \gamma\leq 1$.
Replacing the interpolating $h^t$ with $h_{weak}^t$ alleviates the aforementioned two drawbacks and the restricted functional gradient descent can be proved \citep{grubb2011generalized} to converge to the global optimal under standard regularity conditions.

%% file: notation.tex
\subsection{Notation.} 
For a positive integer $n$, we define $[n] := \{1, 2, \ldots, n\}$. For a vector $x\in\RBB^d$, we use $x_i$ to denote its $i^{th}$ entry. We use $\langle\cdot,\cdot\rangle$ to denote the standard Euclidean inner product in $\RBB^d$ and use $\|\cdot\|$ for the corresponding standard Euclidean~norm.

For some ground set $\XM\subseteq\mathbb{R}^d$, we use $\MM_+^1(\XM)$ to denote the set of probability measures on $\XM$. 
For two measures $\alpha, \beta\in\MM_+^1(\XM)$, we use $TV(\alpha, \beta)$ to denote their total variation distance and use $\mathrm{W}_p(\alpha, \beta)$ to denote their Wasserstein-p distance.
The definition of these two distances are provided in Appendix \ref{section_TV_wasserstein}.
Let $\XM\subseteq\RBB^d$ be an input space. For a fixed distribution $\alpha\in\MM_+^1(\XM)$, we define the corresponding weighted $\LM^p$ space ($1\leq p <\infty$) of functions from $\XM$ to $\RBB^c$ as follows:
\vspace{-2mm}
 $$\LM^p(\alpha) \defi \left\{f:\XM\rightarrow\RBB^c \, \bigg| \,  \left(\int_{\XM}\|f(x)\|^p d \alpha(x)\right)^{\frac{1}{p}}<\infty \right\}.$$
The space $\LM^2(\alpha)$ is endowed with natural inner product and norm: For two functions $f, g \in \LM^2(\alpha)$,
\vspace{-2mm}
\begin{equation*}
	\langle f, g\rangle_{\alpha} = \int_{\XM} \langle f(x), g(x)\rangle d \alpha(x)\ \textrm{and}\ \|f\|_{\alpha} = \sqrt{\langle f, f\rangle_{\alpha}}.
\end{equation*}
In the limiting case when {$p \to \infty$}, we define $$\LM^\infty(\alpha) \defi \left\{f:\XM\rightarrow\RBB^c \,\big|\,\forall x\in\supp(\alpha), \|f(x)\|  <\infty\right\},$$
where $\supp(\alpha)$ denotes the support of $\alpha$. 
We define the $\alpha$-infinity norm of a function $f:\XM\rightarrow\RBB^c$ by $\|f\|_{\alpha, \infty} \defi \sup_{x\in\supp(\alpha)} \|f(x)\|$.
We will also use 
$$\LM^\infty \defi \left\{f:\XM\rightarrow\RBB^c \,\big|\,\forall x\in\XM, \|f(x)\|  <\infty\right\},$$
and define $\|f\|_{\infty} \defi \sup_{x\in\XM} \|f(x)\|$.
We use $\BM^\infty(R)$ to denote the $\LM^\infty$ ball around $0$ with radius $R$.
For a Lipschitz continuous function $f:\XM\rightarrow\RBB^c$, we denote by $\|f\|_\lip$ its Lipschitz constant, i.e. 
\vspace{-2mm}
\begin{equation*}
	\|f\|_\lip \defi \inf \{L:\ \forall x, x' \in \XM,\ \|f(x) - f(x')\| \leq L\|x - x'\|\}.
\end{equation*}
For a function $f:\XM\rightarrow\YM$ and {$R > 0$}, we define the clipping function $\Clip{R}$ to be the projection on $\BM^\infty(R)$, i.e. $\Clip{R}(f)(x) \defi \max\{\min\{f(x), R\}, -R\}.$\\
For a given function $f:\XM\rightarrow \RBB$ and a probability measure $\alpha\in\MM_+^1(\XM)$, their inner product is defined in the standard way: $\langle f, \alpha\rangle = \int_{\XM} f(x) d\alpha(x)$.

%% file: federated_functional_minimization.tex
\section{Federated Functional Minimization} \label{section_problem_formulation}
In this paper, we consider the federated functional minimization problem:
We assume that there are $N$ client machines. Client machine $i$ draws samples from distribution $P_i$ over $\XM \times \YM$. We denote the marginal distribution of $P_i$ on $\XM$ by $\alpha_i$ and the conditional distribution on $\YM$ given $x$ by $\beta_i^x$. Due to the heterogeneous nature of the federated learning problems, the $P_i$'s differ across different clients.

We define $\alpha$ to be the ``arithmetic" mean of the local input probability measures, i.e.\footnote{More precisely, for any Borel set subset to $\XM$, its measure under $\alpha$ is the average of the measures under $\alpha_i's$.}
\begin{equation} \label{eqn_average_measure}
	\alpha \defi \frac{1}{N}\sum_{i=1}^{N}\alpha_i.
\end{equation} 
It is easy to see that $\LM^2(\alpha) \subseteq \LM^2(\alpha_i)$ for all $i$. 
We define $\RM_i: \LM^2(\alpha) \rightarrow \RBB$ to be $\RM_i[f] \defi \EBB_{(x, y) \sim P_i}[\ell(f(x), y)]$ and denote 
$\FM_i[f] = \RM_i[f] + \frac{\mu}{2} \|f\|_{\alpha_i}^2$.
The goal of federated learning is to minimize the average loss functional
\begin{equation} \label{eqn_fed_functional_minimization_2}
	\min_{f\in\CM} \FM[f] = \frac{1}{N}\sum_{i=1}^{N} \FM_i[f],
\end{equation}
where $\CM\subseteq\LM^2(\alpha)$ is some convex set of functions.
In the rest of the paper, we use $f^*$ to denote the optimizer of \eqref{eqn_fed_functional_minimization_2}.
We emphasize that in federated optimization, the local inner product structure $\langle\cdot,\cdot\rangle_{\alpha_i}$ as well as the subgradient \eqref{eqn_subgradient_empirical} \emph{cannot} be shared during training phase.

We consider two notions of client heterogeneity:\\
	 {\bf Semi-heterogeneous setting:} The local distributions on the input space is the same, but the conditional distributions on the output space varies among different machines, i.e. $\forall i \in[N], \alpha_i = \alpha$, but $\beta_i^x$ may not be equal to $\beta_j^x$ for $i\neq j$.\\
	 {\bf Fully-heterogeneous setting:} The local distributions on the input space are different from the global average measure, i.e. $\alpha_i \neq \alpha$;
	 
For the semi-heterogeneous setting, we propose a federated functional gradient boosting method, \ffgd, that provably finds the global optimizer in sublinear time. For the fully-heterogeneous setting, we propose two variants \ffgdc and \ffgdl of \ffgd that converge sublinearly to a neighborhood of the global optimizer. 
We emphasize that in both cases, in order to obtain the best convergence rate, the local machines need to take multiple local steps (specifically $O(\sqrt{T})$ where $T$ is the number of communication rounds) between every round of communication (see Theorems~\ref{thm_ffgd},\ref{thm_ffgdc},\ref{thm_ffgdl}). Such results showing the benefits of taking multiple local steps are rare in the FL literature.

%% file: methodology.tex

\subsection{Semi-heterogeneous Federated Optimization}
\begin{algorithm}[t]
	\caption{Server procedure}
	\begin{algorithmic}[1]
		\Procedure {Server}{$f^0$, $T$, $\CM$}
		\For {$t \leftarrow 0$ to $T-1$}
		\State Sample set $\SM^{t}$ of clients.
		\State $f^{t+1} = \frac{1}{|\SM^{t}|}\sum_{i\in\SM^{t}}\textsc{Client}(i, t, f^{t})$ \label{eqn_server_update}
		\EndFor
		\State \Return $\Proj{\CM}(f^T)$. \label{eqn_return}
		\EndProcedure
	\end{algorithmic}
	\label{algorithm_server}
\end{algorithm}

\begin{algorithm}[t]
	\caption{{Client} procedure for Federated Functional Gradient Descent (\ffgd)}
	\begin{algorithmic}[1]
		\Procedure {Client}{$i$, $t$, $f$}
		\State $\Delta_{i}^{0, t} = 0$, $g_i^{1, t} = f^t$ \label{eqn_zero_initialization};
		\For{$k\leftarrow 1$ to $K$}
		\State $h_{i}^{k,  t}\ := \mathrm{WO}^2_{\alpha_i}(\Delta_{i}^{k-1,  t} + \nabla \RM_i[g_{i}^{k, t}])$ \label{eqn_residual_plus_gradient}
		\State $g_{i}^{k+1, t} := g_{i}^{k, t} - {\eta^{k, t}} {\left(h_{i}^{k,  t} + \mu g_i^{k, t}\right)}$ \label{eqn_update_variable}
		\State $\Delta_{i}^{k,  t} := \Delta_{i}^{k-1,  t} + \nabla \RM_i[g_{i}^{k, t}] - h_{i}^{k,  t}$ \label{eqn_update_residual}
		\EndFor
		\State
		\Return $g_i^{K+1, t}$.
		\EndProcedure
	\end{algorithmic}
	\label{algorithm_ffgd}
\end{algorithm}
In this section, we consider the unconstrained case, i.e. $\CM = \LM^2(\alpha)$.
The \ffgd shares a similar structure as \fedavg. After every round of global variable averaging, the server sends the global consensus function $f^t$ to the workers (see procedure \textsc{Server} in Algorithm \ref{algorithm_server}), which is followed by $K$ local steps of the RFGD update \eqref{eqn_rfgd} tracked via the local variable $g_i^{k, t}$, for $k = 1, 2, \ldots, K$, with the initialization $g_i^{1, t} = f^t$. 
(All the line numbers refer to Algorithm \ref{algorithm_ffgd} in the rest of this paragraph.)
The crucial twist on the RFGD update is the use of an additional residual variable $\Delta_i$, which is initialized to the constant zero function (line \ref{eqn_zero_initialization}). This residual variable accumulates the approximation error of the descent direction $h_i^{k, t}$ incurred by the weak oracle (line \ref{eqn_update_residual}). Such a residual is then used to compensate the next functional subgradient (line \ref{eqn_residual_plus_gradient}) and is used in the query to the weak learning oracle. 
The local variable function $g_i^{k, t}$ is then refined by the approximated functional gradient $\left(h_{i}^{k,  t} + \mu g_i^{k, t}\right)$ (line \ref{eqn_update_variable}).
Note that the residual $\Delta_i$ only tracks the error of estimating $\nabla \RM_i$ since the rest of $\nabla \FM_i$ is exactly available.
After $K$ such updates, the local function $g^{K+1, t}$ is communicated to the server which aggregates these functions across the clients.
Since $\CM = \LM^2(\alpha)$, the \textsc{Server} procedure in Algorithm \ref{algorithm_server} simply returns $f^T$.

The residual $\Delta_i$ used in \ffgd\ (Algorithm \ref{algorithm_ffgd}) is crucial to prove the convergence to the global minimum (see Theorem \ref{thm_ffgd}): in the absence of the residual variable, the error accumulated by the RFGD updates via the calls to the weak learning oracle may not vanish, since in federated learning the local subgradient is \emph{non-zero} even at the global optimal solution due to client heterogeneity. This is in sharp contrast to the single machine case where, at the global optimal solution, the subgradient vanishes and so does the approximation error incurred by the weak oracle, leading to convergence of RFGD. The design of the residual is inspired by the error-feedback technique of the SignSGD method \citep{karimireddy2019error} to mitigate errors in gradient compression in distributed training. This technique has also been applied for functional minimization in the much simpler single machine setting by \citep{grubb2011generalized}.

We now analyze the convergence of \ffgd. 
We make the following standard regularity assumption, which is satisfied e.g. by the cross entropy loss \eqref{eqn_loss_cross_entropy}.
\begin{assumption} \label{ass_bounded_gradient}
	For all $i \in [N]$, the subgradients of $\RM_i[f]$ are $G$-bounded under the $\LM^2(\alpha_i)$ norm, i.e. for any $f\in\LM^2(\alpha)$, we have $\|\nabla \RM_i[f]\|_{\alpha_i} \leq G$.
\end{assumption}


Under this assumption, we have the convergence result for \ffgd. For simplicity, in this section and subsequent ones for \ffgdc and \ffgdl, we present the convergence rate in the case when all $N$ clients are sampled in each round. The analysis easily extends to the case when only a few clients are sampled in each round, incurring a penalty for the variance in the sampling. Qualitatively the convergence bound stays the same. Details can be found in the appendix.
\begin{theorem}[Convergence result of \ffgd] \label{thm_ffgd}
	Let $f^0$ be the global initializer function. Suppose that Assumption \ref{ass_bounded_gradient} holds, and suppose the weak learning oracle $\mathrm{WO}_\alpha^2$ satisfies \eqref{eqn_weak_oracle_contraction} with constant $\gamma$.
	Using the step size $\eta^{k, t} = \frac{2}{\mu(tK+k+1)}$, the output of \ffgd satisfies
	\begin{equation*}
		\begin{aligned}
			\| f^{T} - f^*\|_\alpha^2&\ = O\Bigg(\frac{\| f^{0} - f^*\|_\alpha^2}{KT} + \frac{KG^2\log (KT)}{T\gamma^2\mu^2} \\
			&+ \frac{(1-\gamma)G^2}{K\mu^2\gamma^2} + \frac{(1-\gamma)G^2\log (KT)}{KT\mu\gamma^2}\Bigg).
		\end{aligned}
	\end{equation*}
\end{theorem}
In the limit case when $\gamma=1$, the weak oracle exactly approximates its input and hence the above result degenerates to the one of \fedavg \citep{li2019convergence}: $$\| f^{T} - f^*\|_\alpha^2 = O\left(\frac{\| f^{0} - f^*\|_\alpha^2}{KT} + \frac{KG^2\log (KT)}{T\mu^2}\right).$$
Note that when $\gamma<1$, in order to have the best convergence rate (up to log factors), we need to set the number of local steps $K = \Omega(\sqrt{T})$, which leads to the following corollary.
\begin{corollary}
	Under the same conditions as Theorem \ref{thm_ffgd}, when $\gamma < 1$, the best convergence rate (up to log factors) is $\|f^{T} - f^*\|_\alpha^2 = O(\frac{1}{\sqrt{T}})$ by choosing $K = \Omega(\sqrt{T})$.
\end{corollary}

\subsection{Fully-heterogeneous Federated Optimization}
\begin{algorithm}[t]
	\caption{\textsc{Client} procedure for Federated Functional Gradient Descent with Clipping (\ffgdc)}
	\begin{algorithmic}[1]
		\Procedure {Client}{$i$, $t$, $f$}
		\State $\Delta_{i}^{0, t} = 0$, $g_i^{1, t} = \Clip{B}(f^t)$ \label{eqn_zero_initialization_c}
		\For{$k\leftarrow 1$ to $K$}
		\State $h_{i}^{k, t}\ := \mathrm{WO}_{\alpha_i}^\infty(\Delta_{i}^{k-1, t} + \nabla \RM_i[g_{i}^{k, t}])$ \label{eqn_residual_plus_gradient_c}
		\State $g_{i}^{k+1, t} := g_{i}^{k, t} - {\eta^{k, t}}\cdot \Clip{ G^2_\gamma}\left(h_{i}^{k, t}\right)$ \label{eqn_update_variable_c}
		\State $\Delta_{i}^{k, t} := \Clip{G^1_\gamma}\left(\Delta_{i}^{k-1, t} + \nabla \RM_i[g_{i}^{k, t}] - h_{i}^{k, t}\right)$ \label{eqn_update_residual_c}
		\EndFor
		\State
		\Return $g_i^{K+1, t}$.
		\EndProcedure
	\end{algorithmic}
	\label{algorithm_ffgdc}
\end{algorithm}
The fact that \ffgd is able to find the global optimal solution heavily relies on the consensus of the local inner product structure: $\langle\cdot ,\cdot \rangle_\alpha = \langle\cdot ,\cdot \rangle_{\alpha_i}$,
which is absent in the fully-heterogeneous setting. So we can only guarantee that the algorithm finds a solution within a neighborhood of the global optimal whose radius diminishes to zero with diminishing heterogeneity of the $\alpha_i$'s.

To achieve this goal, we consider minimizing the objective functional $\mathcal{F}[f]$ of problem \eqref{eqn_fed_functional_minimization_2} over the ball $\BM^\infty(B)$ for a given radius $B > 0$, i.e. $\CM = \BM^\infty(B)$.

This feasible domain allows us to exploit the variational formulation of the total variation (TV) distance in order to relate the local inner product structure to the global one:
\begin{equation} \label{eqn_change_inner_product_TV}
	|\langle f, g\rangle_{\alpha_i} - \langle f, g\rangle_{\alpha}| \leq 2\|f\|_\infty\|g\|_\infty TV(\alpha, \alpha_i).
\end{equation}

We now present \ffgdc,  a variant of \ffgd with additional clipping operations, in Algorithm \ref{algorithm_ffgdc} (the constants $G^1_\gamma$ and $G^2_\gamma$ are specified in Theorem \ref{thm_ffgdc}). Only the \textsc{Client} procedure is shown; the \textsc{Server} procedure is identical to that of \ffgd. 
\ffgdc\ relies on a different weak learning oracle $\mathrm{WO}_{\alpha_i}^\infty$ that satisfies an $\LM^\infty(\alpha_i)$ version of \eqref{eqn_weak_oracle_contraction}. Specifically, for any query $h\in\LM^\infty(\alpha_i)$, we assume that
\begin{equation} \label{eqn_weak_oracle_contraction_c}
	\|\mathrm{WO}_{\alpha_i}^\infty(h) -  h\|_{\alpha_i, \infty} \leq (1-\gamma)\|h\|_{\alpha_i, \infty},  
\end{equation}
for some positive constant $0< \gamma\leq 1$.

The clipping step in line \ref{eqn_zero_initialization_c} of the \textsc{Client} procedure actually implements the global projection to $\CM$ as $f^t$ is consensus among all machines.
Together with the clipping steps in lines \ref{eqn_residual_plus_gradient_c} and \ref{eqn_update_residual_c} in the \textsc{Client} procedure, these operations ensure that during the entire optimization procedure, the (local and global) variable functions are uniformly bounded\footnote{Note that such a boundedness property is non-trivial even when the weak oracle only returns bounded functions. This is because the $\LM^\infty$ norm of the variable function potentially diverges as $t \rightarrow \infty$ since standard choices of step sizes are not summable.
} in the whole domain $\XM$.
Therefore, the bound \eqref{eqn_change_inner_product_TV} can be exploited. Another merit of the clipping step is its non-expansiveness since it is exactly the projection operator onto the $\LM^\infty$ ball. This is important to our analysis.

We next analyze the convergence of \ffgdc in the fully-heterogeneous setting. 
We need a slightly strengthened version of Assumption \ref{ass_bounded_gradient}: 
\begin{assumption} \label{ass_bounded_gradient_c}
	The subgradients of $\RM_i[f]$ are $G$-bounded under the $\LM^\infty(\alpha_i)$ norm, i.e.
	\begin{equation*}
		\forall f\in \LM^\infty(\alpha_i), \|\nabla \RM_i[f]\|_{\alpha_i, \infty} \leq G.
	\end{equation*}
\end{assumption}
The following theorem states that \ffgdc\ converges to a neighborhood of the global minimizer, with a radius $r$ proportional to the average TV distance between $\alpha$ and $\alpha_i$, i.e. $r = O(\frac{1}{N}\sum_{i=1}^{N} \TV(\alpha, \alpha_i))$.
\begin{theorem}[Convergence result of \ffgdc]\label{thm_ffgdc}
	Let $f^0$ be the global initializer function. Let $\omega = \frac{1}{N}\sum_{i=1}^{N} \TV(\alpha, \alpha_i)$. Set $G_\gamma^1 = \frac{1-\gamma}{\gamma}\cdot G$ and $ G_\gamma^2 = \frac{2-\gamma}{\gamma}\cdot G$.
	Under Assumption \ref{ass_bounded_gradient_c}, and supposing the weak learning oracle $\mathrm{WO}_\alpha^\infty$ satisfies \eqref{eqn_weak_oracle_contraction_c} with constant $\gamma$, using the step sizes $\eta^{k, t} = \frac{4}{\mu(tK+k+1)}$, the output of \ffgdc satisfies
	\begin{equation*}
		\begin{aligned}
			\|\Clip{B}({f^T}) - &\ f^*\|_\alpha^2 = O\bigg({\frac{\|f^0 - f^*\|_\alpha^2}{KT}} + {\frac{KG^2 log(KT)}{T\mu^2\gamma^2}} \\
			&\ + \frac{(1-\gamma)^2G^2}{K\mu^2\gamma^2} + {\frac{GB\omega}{\mu\gamma}} + {\frac{G^2\log(TK)\omega}{\gamma^2\mu^2 T}}\bigg)
		\end{aligned}
	\end{equation*}
\end{theorem}
As previously discussed, the key to prove the above theorem is to ensure that the variable function remains bounded during the entire optimization procedure. Moreover, the choice of the constants $ G_\gamma^1$ and $ G_\gamma^2$ ensures that the clipping operations do not affect the value of $\Delta_i^{k, t}$ and $h_i^{k, t}$ on the support of $\alpha_i$.
The existence of these two constants is due to the stronger weak learner oracle \eqref{eqn_weak_oracle_contraction_c}: while the standard $\LM^2(\alpha)$ oracle \eqref{eqn_weak_oracle_contraction} ensures the residual is reduced in average, it may still have large spiky values on $\supp(\alpha_i)$.

In contrast to the semi-heterogeneous setting, even in the ideal case when $\gamma = 1$, \ffgdc does not converge to the global minimum due to the forth term in Theorem \ref{thm_ffgdc}. This is because the fully-heterogeneous setting is fundamentally harder: local strong convexity due to Tikhonov regularization does not imply global strong convexity in the fully-heterogeneous setting which hinders the convergence to the global minimum of \ffgdc.


\subsection{Federated Functional Least Squares Minimization} \label{section_fflsm}
\begin{algorithm}[t]
	\caption{\textsc{Client} procedure for Federated Functional Gradient Descent for $\ell_2$ Regression among Lipschitz Continuous Functions (\ffgdl)}
	\begin{algorithmic}[1]
		\Procedure {Client}{$i$, $t$, $f$}
		\State $\Delta_{i}^0 = 0$, $g_i^{1, t} = f^t$ \label{eqn_zero_initialization_L};
		\For{$k\leftarrow 1$ to $K$}
		\State $h_{i}^{k}\ := \mathrm{WO}^\lip_{\alpha_i}(\Delta_{i}^{k-1} - u_i)$ \label{eqn_residual_plus_gradient_L};
		\State $g_{i}^{k+1, t} := g_{i}^{k, t} - {\eta^{k, t}} (g_{i}^{k, t} - h_{i}^{k})$ \label{eqn_update_variable_L};
		\State $\Delta_{i}^{k} := \Delta_{i}^{k-1} - u_i + h_{i}^{k}$; \label{eqn_update_residual_L}
		\EndFor
		\State
		\Return $g_i^{K, t}$.
		\EndProcedure
	\end{algorithmic}
	\label{algorithm_ffgd_L}
\end{algorithm}
In this section, we show that when the loss $\ell$ is the square loss $\ell(y', y) = \frac{1}{2}\|y' - y\|^2$ for $y, y'\in\YM\subseteq\RBB^c$, we can improve the bound on radius of convergence $r$ by replacing the TV distance with the Wasserstein-1 distance, if the feasible domain is further restricted to the family of Lipschitz continuous functions. 
Note that even in the semi-heterogeneous case, the square loss requires special treatment as it does not satisfy Assumption \ref{ass_bounded_gradient}.

For simplicity, we assume that the output domain $\YM\subseteq\RBB$, the extension to $\RBB^c$ is analogous. For some parameter $L > 0$ and a constant $\gamma$ (defined in \eqref{eqn_weak_oracle_contraction_l2}), we consider the federated functional least squares minimization problem
\begin{equation} \label{eqn_fed_functional_minimization_L}
	\min_{\|f\|_\lip\leq L/\gamma} \FM[f] = \frac{1}{N}\sum_{i=1}^{N}\EBB_{(x,y)\sim P_i}[\tfrac{1}{2}(f(x) - y)^2].
\end{equation}
The domain of optimization is the set $\CM = \{f:\ \|f\|_\lip\leq L/\gamma\}$. This feasible domain allows us to to exploit the variational formulation of the Wasserstein-1 ($\mathrm{W}_1$) distance in order to relate the local inner product structure to the global one: Denote $\xi = \|f\|_{\lip}\|g\|_{\alpha, \infty} + \|g\|_{\lip}\|f\|_{\alpha, \infty}$,
\begin{equation} \label{eqn_W1_variational_formulation}
	|\langle f, g\rangle_{\alpha_i} - \langle f, g\rangle_{\alpha}| \leq \xi \mathrm{W}_1(\alpha, \alpha_i).
\end{equation}

We assume in this section that $P_i$ is the empirical measure on a finite set of client data $\{(x_{i,j}, y_{i,j}):\ j \in [M]\}$. We also assume that the data satisfies the following Lipschitzness property: for any $j, j' \in [M]$, we have 
\begin{equation} \label{eqn_lipschitz_data}
|y_{i,j} - y_{i,j'}| \leq L\|x_{i,j} - x_{i,j'}\|.
\end{equation}
This assumption implicitly suggests that the labels are generated from the inputs via an $L$-Lipschitz function with no additive noise. The assumption implies that it is possible to construct an $L$-Lipschitz function $u_i: \RBB^d \to \RBB$ that interpolates the data $\{(x_{i,j}, y_{i,j}):\ j \in [M]\}$: specifically, this Lipschitz extension is defined as
\begin{equation} \label{eqn_lipschitz_extension}
	u_i(x) \defi \min_{j \in [M]}\left(y_{i,j} + L\|x-x_{i,j}\| \right).
\end{equation}
It is easy to check that $u_i$ is an $L$-Lipschitz function such that $u_i(x_{i,j}) = y_{i,j}$ for all $j \in [M]$. If the output domain $\YM$ is high dimensional, i.e. $\YM \subseteq \RBB^c$, then the construction of $u_i$ follows from Kirszbaum's Lipschitz extension theorem~\citep{schwartz1969nonlinear}. The function $u_i$ is used in the \textsc{Client} procedure in Algorithm \ref{algorithm_ffgd_L} which describes a variant of \ffgd named \ffgdl. It is designed to solve the federated least squares minimization problem over the Lipschitz continuous functions. The \textsc{Server} procedure for \ffgdl is identical to that of \ffgd.
Note that while $\CM$ is not the whole $\LM^2(\alpha)$ space, \ffgdl ensures the feasibility of $f^T$ and hence the projection step in the \textsc{Server} procedure simply returns $f^T$.


\ffgdl differs from \ffgd as it exploits the following observation: for problem \eqref{eqn_fed_functional_minimization_L}, the functional subgradient at $g$ can be computed is any function $h$ satisfying
\begin{equation} \label{eqn_subgradient_l2_regression}
	\forall j\in[M],\ h(x_j) = g(x_{i, j}) - y_{i,j}.
\end{equation}
In particular, the function $h = g - u_i$ is one such subgradient.
Thus, given a function $g$, to approximate $g - u_i$, it suffices to approximate $u_i$. In \ffgdl, we use the weak learning oracle to approximate $u_i$ in \eqref{eqn_subgradient_l2_regression} instead of approximating the whole subgradient like \ffgd and \ffgdc. Note that \ffgdl uses a higher order weak learner oracle $\mathrm{WO}_{\alpha_i}^l$ which will be described momentarily. 
The function $h_i^{k}$ represents the approximation to $u_i$ after $k$ steps in the \textsc{Client} procedure. We use $(g^{k, t}_i - h_i^{k})$ as descent direction (line \eqref{eqn_update_variable_L}) and the residual variable $\Delta_i^k$ accumulates the error incurred from approximating $u_i$ (lines \eqref{eqn_residual_plus_gradient_L} and \eqref{eqn_update_residual_L}). We emphasize that the subgradient \eqref{eqn_subgradient_l2_regression} cannot be directly used to update the local variable since $u_i$ involves the data points.

Now we analyze the convergence of \ffgdl. 
We first clarify the weak oracle $\mathrm{WO}_{\alpha}^\lip$ required by \ffgdl: on any query $\phi\in\LM^2(\alpha)$ with $\|\phi\|_\lip<\infty$, $\mathrm{WO}_{\alpha}^\lip$ outputs a function $h$ such that $\|h\|_\lip\leq\|\phi\|_\lip$ and the following two conditions hold simultaneously
\begin{align}
	\|h -  \phi\|_{\alpha, \infty} \leq&\ (1-\gamma)\|\phi\|_{\alpha, \infty},  \label{eqn_weak_oracle_contraction_l1}\\
	\|h -  \phi\|_\lip \leq&\ (1-\gamma)\|\phi\|_\lip,  \label{eqn_weak_oracle_contraction_l2}
\end{align}
for some positive constants $0< \gamma\leq 1$.
We can implement this oracle using the Sobolev training \citep{czarnecki2017sobolev}.
{This is further discussed in the appendix.}

Our analysis needs the following boundedness assumptions:
\begin{assumption} \label{ass_bounded_value_h}
	For some parameter $B > 0$, all labels are $B$-bounded: i.e. $\forall i\in[N], j\in[M], -B\leq y_{i,j} \leq B$.\\
	For every pair of measures $\alpha_i$ and $\alpha_{i'}$,  $\forall x_{i,j}\in\supp(\alpha_i)$, $\exists x_{i',j'}\in\supp(\alpha_{i'})$ such that $\|x_{i, j} - x_{i',j'}\|\leq D$.
\end{assumption}

We now present the convergence result of \ffgdl.
\begin{theorem}[Convergence result of \ffgdl]\label{thm_ffgdl}
	Let $f^0$ be the global initializer. Let $G^2 = \frac{2L^2}{N^2}\sum_{i, s=1}^{N} \mathrm{W}_2^2(\alpha_s, \alpha_i) + 2B^2$ and  $\omega = \frac{1}{N}\sum_{i=1}^{N}\mathrm{W}_1(\alpha, \alpha_i)$.
	Consider the federated functional least square minimization problem \eqref{eqn_fed_functional_minimization_L}.
	Under Assumption \ref{ass_bounded_value_h}, and supposing the weak learning oracle $\mathrm{WO}_\alpha^{\lip}$ satisfies \eqref{eqn_weak_oracle_contraction_l1} and \eqref{eqn_weak_oracle_contraction_l2} with constant $\gamma$, using the step sizes $\eta^{k,t} = \frac{4}{\mu(tK+k+1)}$, the output of \ffgdl satisfies
	\begin{equation*}
				\begin{aligned}
		\|f^T - f^*\|_\alpha^2 = O\Bigg({\frac{\|f^0 - f^*\|^2}{KT}} + {\frac{KG^2 log(KT)}{T\mu^2\gamma^2}} \\+ {\frac{(1-\gamma)^2B^2}{\mu^2\gamma^2 K}}+ \frac{(DL^2+ BL)\omega}{\gamma^2\mu}\Bigg).
				\end{aligned}
	\end{equation*}
\end{theorem}
The key component in the analysis of \ffgdl\ is to ensure that the variable function remains Lipschitz continuous along the entire optimization trajectory.
However, unlike \ffgdc\ where the boundedness of the variable is a direct consequence of the clipping operation, maintaining the Lipschitz continuity of the variable in \ffgdl is more subtle and heavily relies on the structure of the update rule in line \eqref{eqn_update_variable_L}. We elaborate this in the appendix.

\begin{remark}
	While we present \ffgdl in a structure similar to \fedavg, in fact, it can be implemented using only one round of communication: Lines \ref{eqn_residual_plus_gradient_L} and \ref{eqn_update_residual_L} are independent of $t$. Therefore, we can compute $\{h_i^k\}_{k=1}^K$ once and send them to the server, where the update and the aggregation of the function ensembles are actually carried out.
\end{remark}

%% file: Experiment.tex
\section{Experiments} \label{section_experiment}
We conduct experiments on the two  datasets CIFAR10 and MNIST and compare with \fedavg as baseline to showcase the advantage of our functional minimization formulation and the proposed algorithms.
We will  investigate the effect of the number of local steps, number of communication rounds, communication cost, as well as effect of the residual technique used in our algorithms.

\subsection{Results on CIFAR10}
\paragraph{Accuracy vs Communication Cost.}
To make a fair comparison between \ffgd\ (which is a non-parametric method) and \fedavg\ (which is used for  parametric models), we use the number of models exchanged (or equivalently, the number of parameters exchanged) to measure the communication cost.
Recall that $N$ and $K$ are the number of local workers and local steps, respectively.
The per-iteration cost of \ffgd is $NK$: A client needs to upload its local increment which is an ensemble of $K$ models and it has to download $(N-1)K$ models from the server.
For \fedavg, this quantity is $2$, i.e. the client uploads its own model and downloads the shared model from the server. Note that the models have the same number of parameters. 
\begin{figure}[t]
	\begin{tabular}{c@{}c}
		\includegraphics[width=.49\columnwidth]{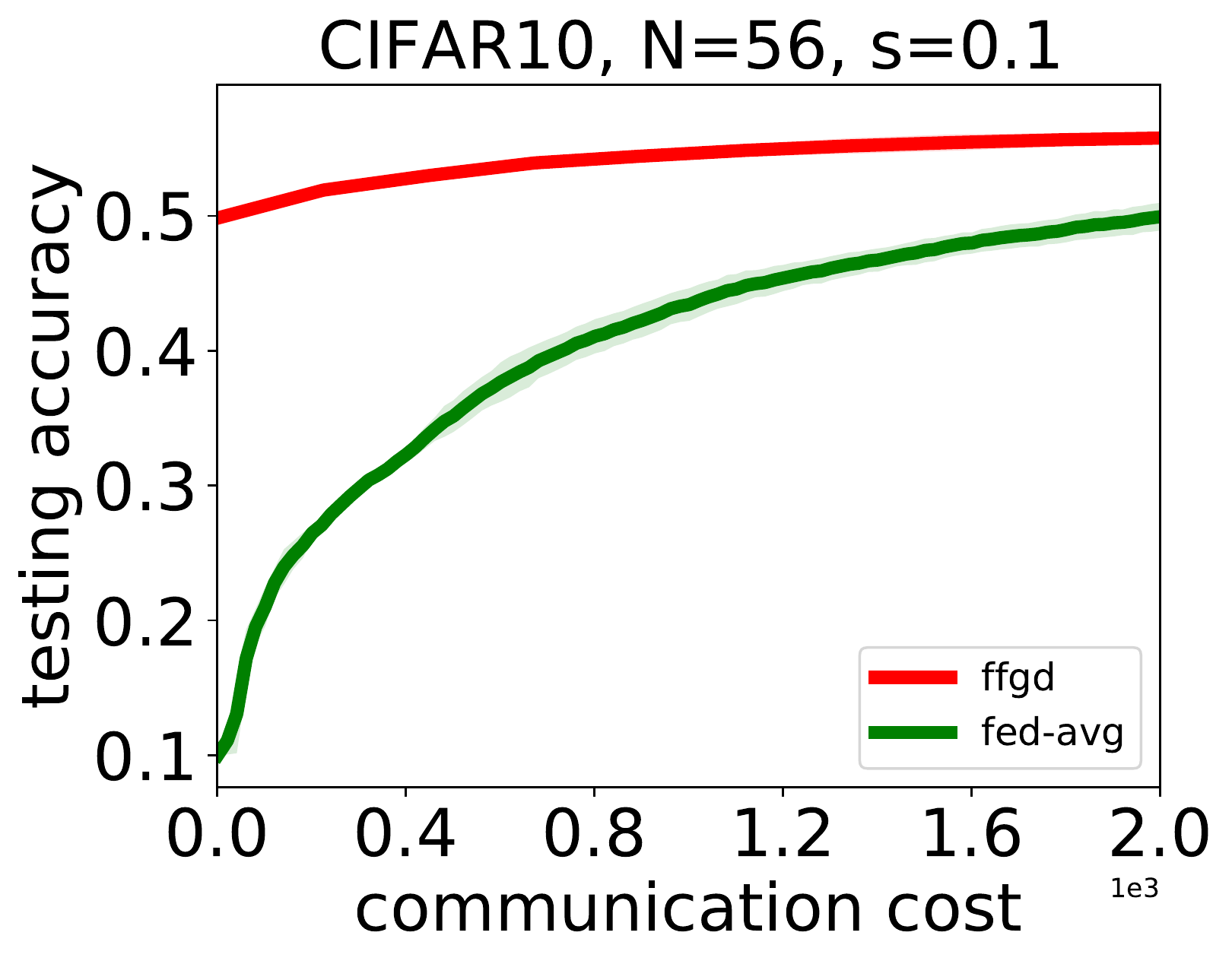} &
		\includegraphics[width=.49\columnwidth]{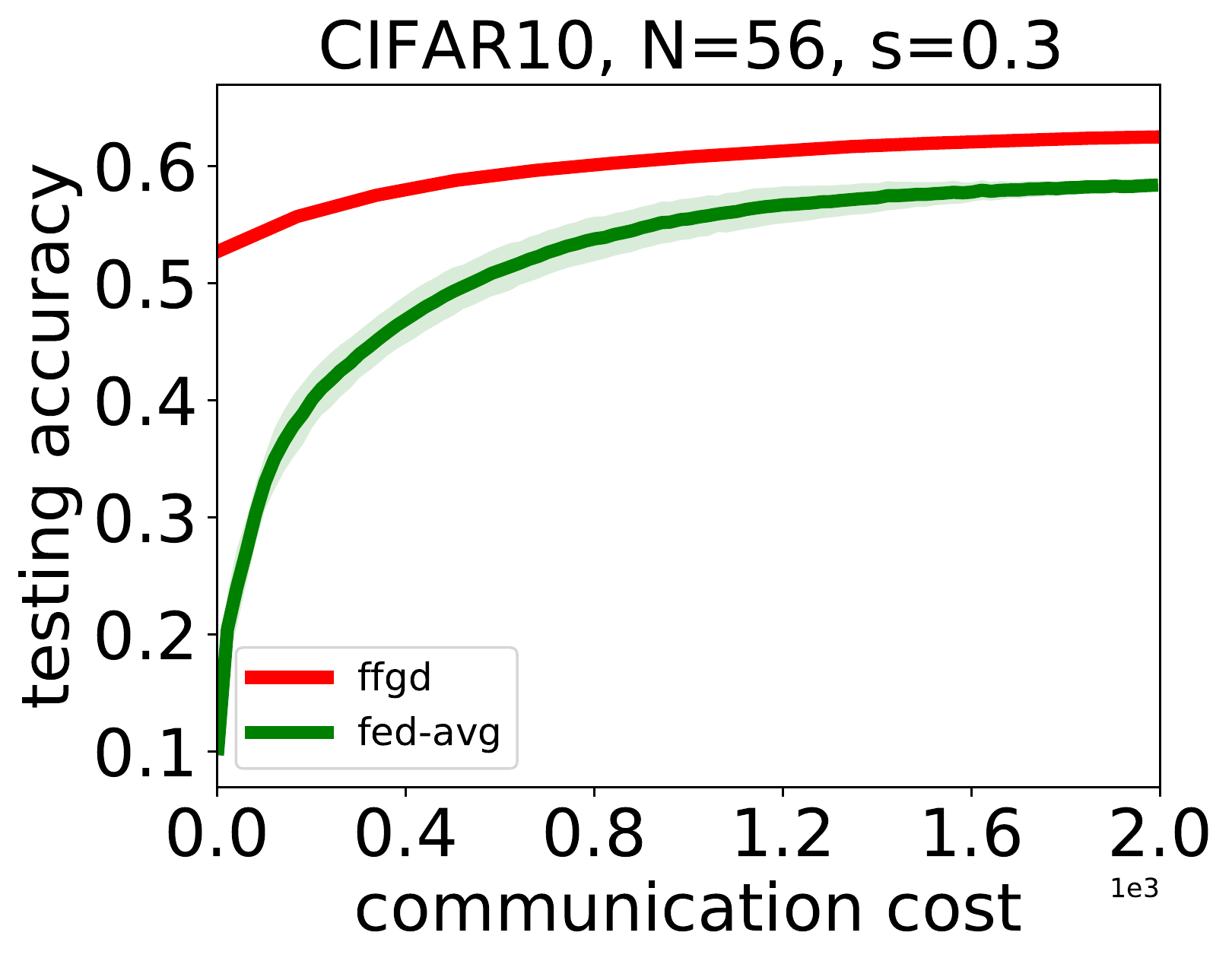}
	\end{tabular}
	\caption{Testing accuracy on CIFAR10.}
	\label{fig_obs1_cifar10}
\end{figure}
\paragraph{Setup.}
In our experiments, we consider the multiclass image classification problem  using the cross entropy loss given in \eqref{eqn_loss_cross_entropy}.
The class of weak learners is a CNN similar {to the one suggested by the PyTorch tutorial for CIFAR10}.
The heterogeneity across local datasets is controlled by dividing the dataset among $N=56$ clients in the following manner: we randomly select a portion $s\times100\%$ ($s\in[0, 1]$) of the data from the dataset and allocate them equally to all clients; for the remaining $(1-s)\times100\%$ portion of the data, we sort the data points by their labels and assign them to the workers sequentially.
This is a same scheme as employed in \citep{karimireddy2020scaffold, hsu2020federated} to enforce heterogeneity.
In \fedavg, each worker takes $10$ local steps and the step sizes are set to constants $5\times10^{-4}$ (a larger step size leads to divergence).
\begin{figure}[t]
	\begin{tabular}{c c}
		\includegraphics[width=.45\columnwidth]{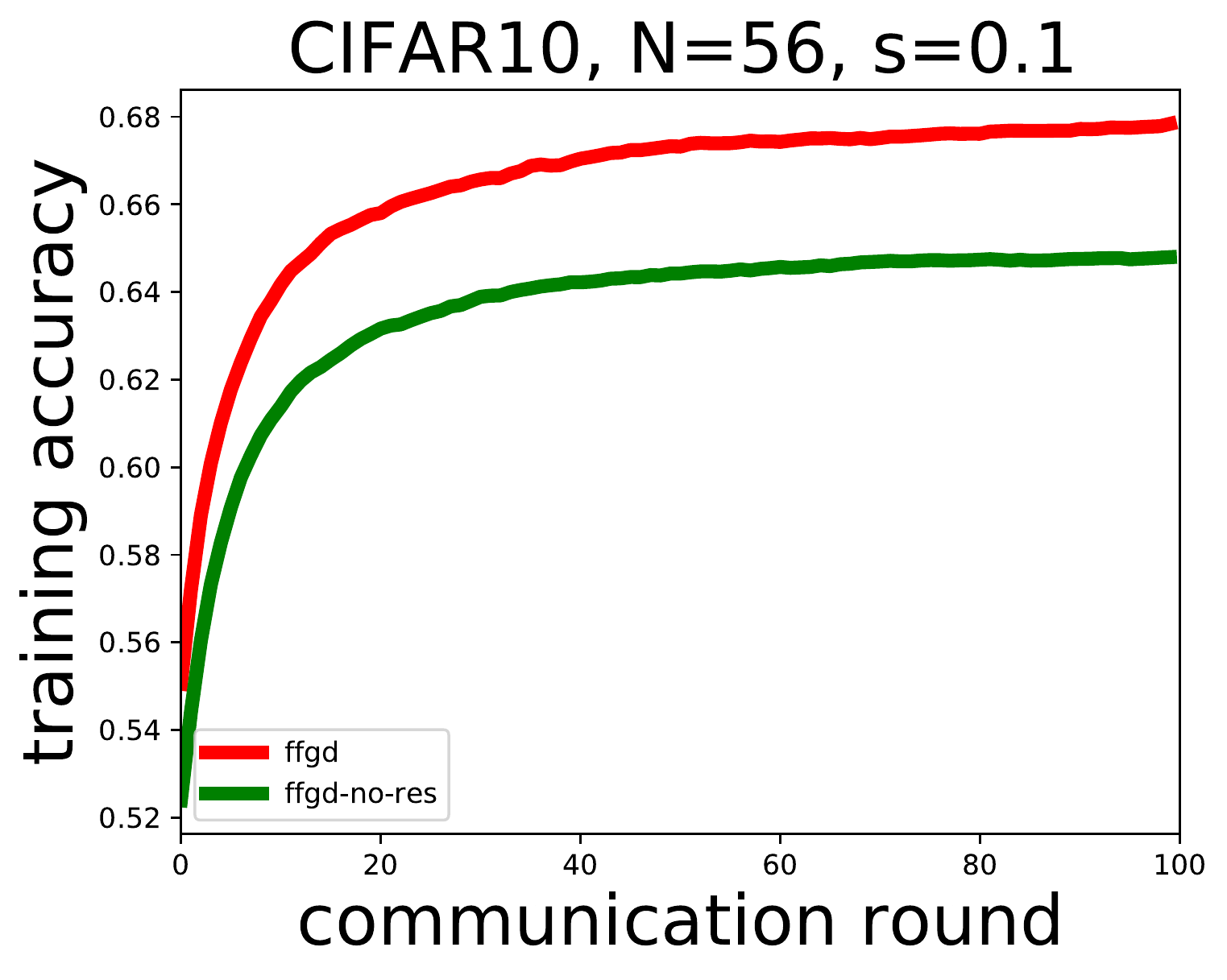} &
		\includegraphics[width=.45\columnwidth]{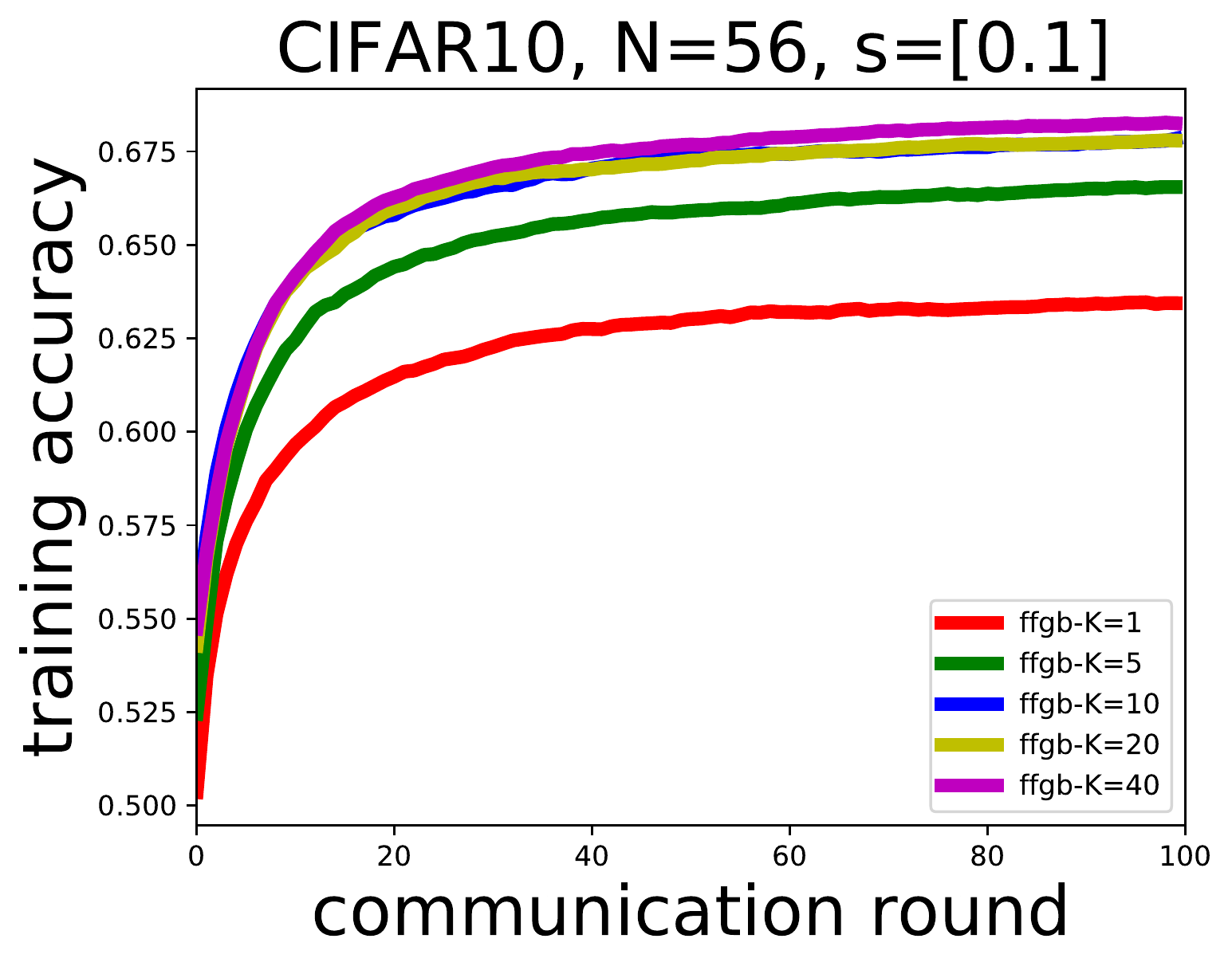}  \\
		(a) & (b)
	\end{tabular}
	\caption{(a) Comparison of \ffgd\ with a variant without the residual step; (b)  Effect of the number of local steps in FFGB.}
	\label{fig_obs2_cifar10}
\end{figure}
\paragraph{Comparing with \fedavg.}
We present the results corresponding to CIFAR10 in Figure \ref{fig_obs1_cifar10}, where each figure corresponds to a different heterogeneity setting. We consider  two values for $s$, i.e. $s=0.1,0.3$.
For \ffgd, the number of local steps $K$ is $4, 3$ for $s=0.1, 0.3$, respectively.
A larger $K$ leads to faster convergence rate according to our theory, however it also leads to higher communication complexity since the per-iteration communication cost of \ffgd\ is $NK$.
The step size follows the scheme $\eta^{k, t} = \frac{\eta_0}{K t+k+1}$, where $\eta_0$ is set to $10$, $20 $ respectively for $s=0.1, 0.3$.
In this experiment, we limit the total communication cost to be $2000$ for \ffgd\ and \fedavg\ (this corresponds to $1000$ rounds for \fedavg\ and $2000/N/K$ rounds for \ffgd).
As we observe from Figure~\ref{fig_obs1_cifar10},  even with the current  implementation of \ffgd\ which is quite communication-inefficient (i.e. it has a much high communication cost per round compared to \fedavg\ ), \ffgd\ achieves a higher accuracy than \fedavg. 
We acknowledge that when $N$ grows, a direct implementation of \ffgd is non-ideal due to the high per-iteration communication complexity.
We will  discuss in Section~\ref{section_distillation} ways to improve the communication cost of  \ffgd\  using knowledge distillation methods.

\paragraph{Residuals are necessary.}
We compare \ffgd with its no residual variant to show that the augmented residual is necessary to ensure a fast convergence.
For both methods, we set the local steps $K$ to $10$ and set the initial step size $\eta_0$ to $10$ for $s=0.1$.
We present the result in Figure \ref{fig_obs2_cifar10}(a) where we observe that \ffgd\ consistently achieves a higher training accuracy over the no-residual counterpart.
\paragraph{Local steps are necessary.}
To show  the necessity of the local steps, we vary the number of local steps $K\in\{1, 5, 10, 20, 40\}$ in \ffgd\ and compare the 
in Figure \ref{fig_obs2_cifar10}(b). We observe that a larger $K$ leads to higher training accuracy.

\subsection{Results on MNIST}
In Figure \ref{fig_obs1_mnist} we report the testing accuracy of the two algorithms  \ffgd\ and \fedavg in terms of  the number of \emph{communication rounds} for the MNIST data set.
As we observe from the figure,   \ffgd\  performs significantly superior w.r.t.  \fedavg under the same number of communication rounds. This shows how powerful functional minimization methods can be.  Recall that the per-round communication complexity of  
\ffgd\ is higher than \fedavg. We will provide the comparison w.r.t the communication cost in the appendix (with similar results as the previous section).

\begin{figure}[t]
	\begin{tabular}{cc}
		\includegraphics[width=.45\columnwidth]{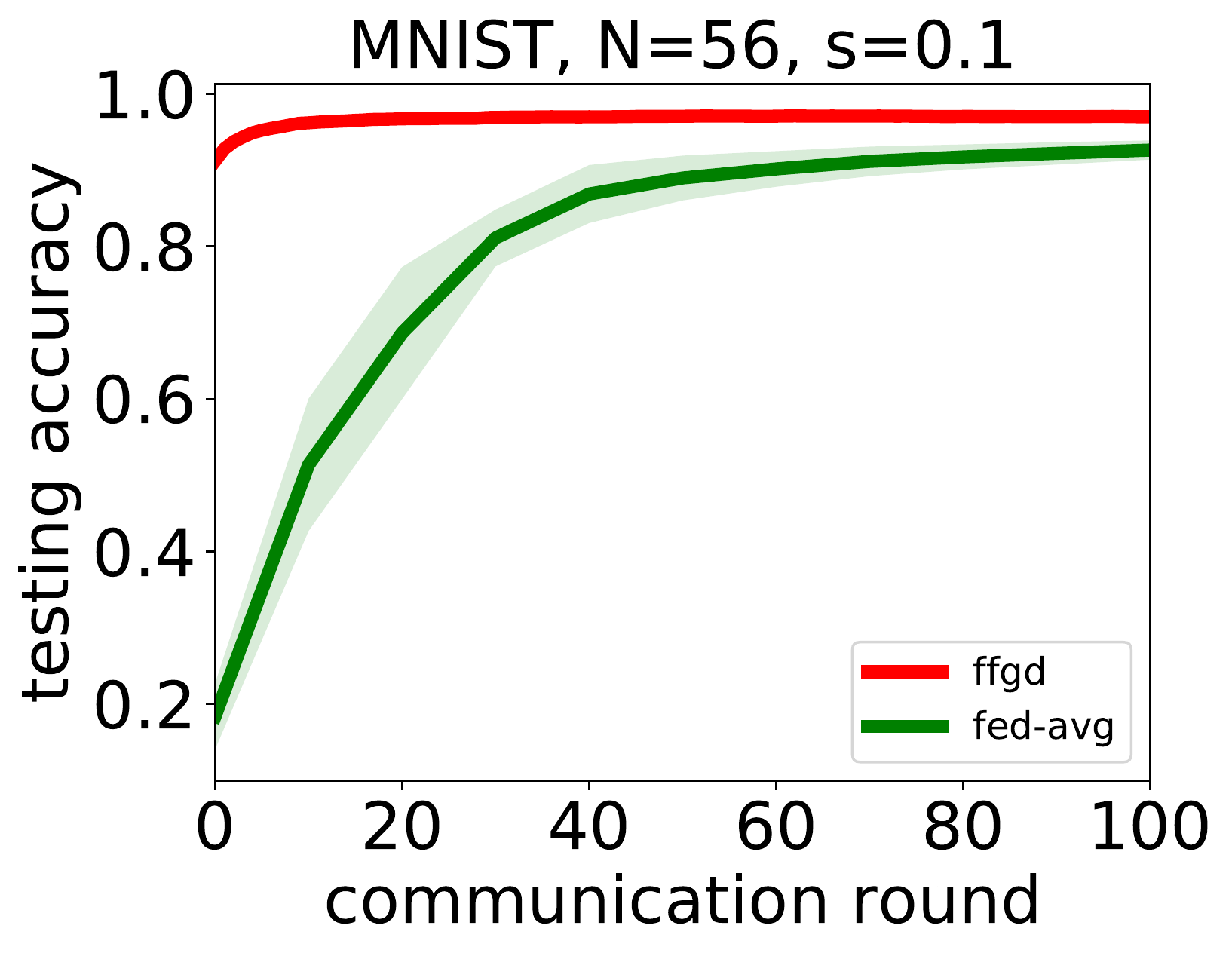} &
		\includegraphics[width=.45\columnwidth]{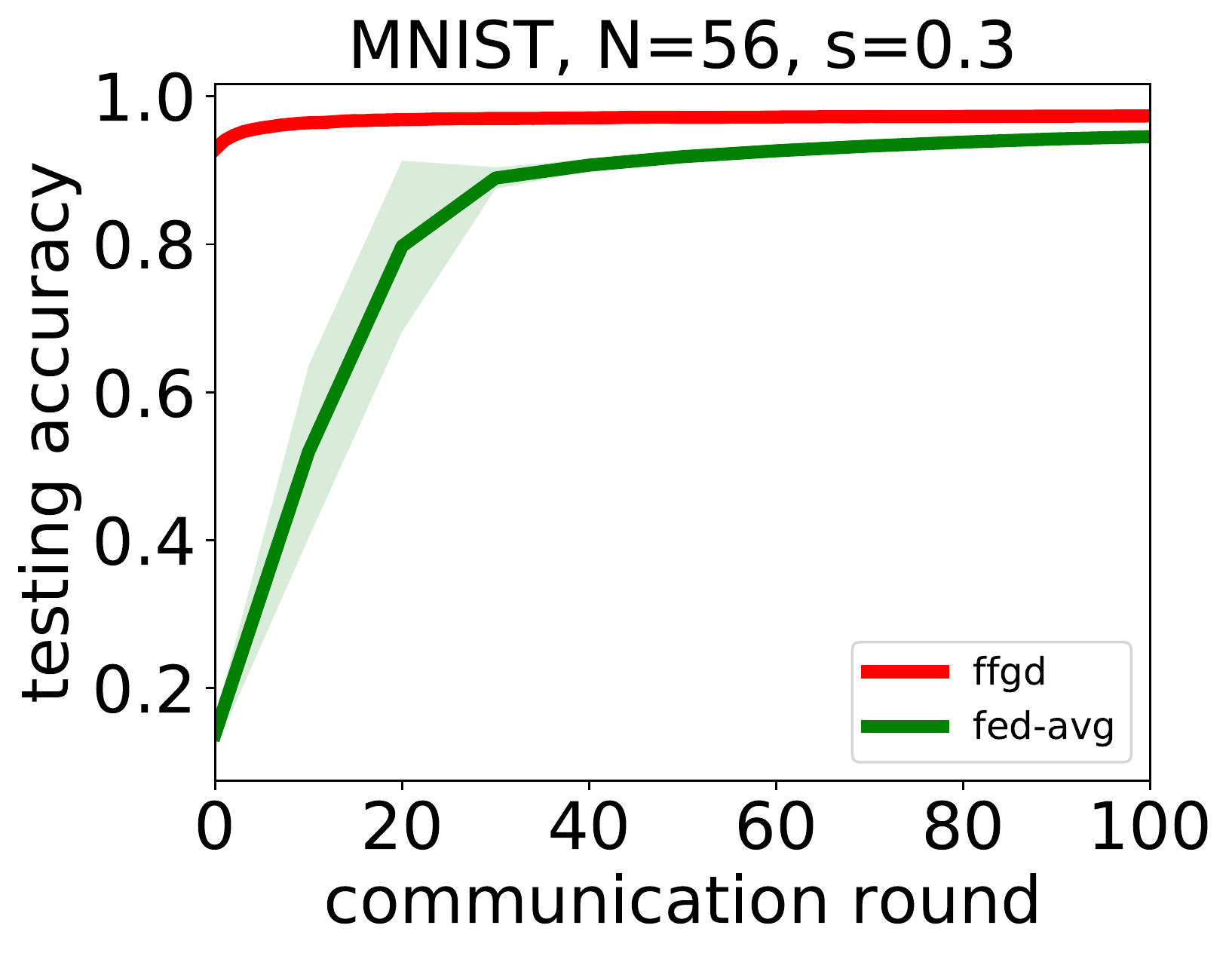}
	\end{tabular}
	\caption{Testing accuracy vs the number of communication rounds on MNIST.}
	\label{fig_obs1_mnist}
\end{figure}

%% file: practicalities.tex

\section{Towards a practical implementation} \label{section_distillation}


A direct implementation of the algorithms in this paper has one significant drawback in practice: it requires aggregation of a very large number of models. In each of the $T$ communication rounds, each of the $N$ clients computes $K$ functions that are communicated to the server, which aggregates all of them. Thus, the final ensemble has $TNK$ models, which is impracticably large. 

The focus of this paper is on developing the theory of federated functional minimization algorithms and the development of practical versions of our algorithms is a very important direction of future work. However, we now indicate how the technique of {\em knowledge distillation} \citep{hinton-distillation,caruana} can be used to reduce the final ensemble to a more reasonable size. Knowledge distillation converts an ensemble of models into a single model by training a new model to mimic the ensemble's predictions. This training procedure only needs access to {\em unlabelled} examples drawn from the appropriate marginal distribution. \citet{hinton-distillation} have shown this technique to be very effective (i.e. incurring minimal loss in performance) in practice when we have access to a large number of unlabelled examples. 

The implication for our algorithms is that since clients have access to their local distributions $P_i$, they can distill the ensemble of $K$ models they compute in each round to a single model sent to the server. The server simply averages the $N$ models it receives from the clients, so the final ensemble size is reduced to $N$. This can be further reduced to a \emph{single} model if we make the reasonable assumption that the server has access to the mean distribution $\alpha = \frac{1}{N}\sum_{i=1}^N \alpha_i$ on the feature vectors, because then the server can distill the aggregation of the $N$ models it receives from the clients down to a single model. Even if the server doesn't have access to $\alpha$, it is still possible to reduce the ensemble size down to a single model by interleaving distillation rounds between boosting rounds. Specifically, after every boosting round, the server executes a distillation round via \fedavg. Distillation via \fedavg is possible the clients simply need to compute the gradients of the loss incurred by the single model under training on their own local data.

%% file: appendix.tex
\section{Total Variation Distance and Wasserstein-1 Distance between Probability Measures} \label{section_TV_wasserstein}
Given two probability distributions $\alpha, \beta\in\MM_+^1(\XM)$, the total variation distance between $\alpha$ and $\beta$ is
\begin{equation}
	\TV(\alpha, \beta) = \sup_{A \in \FM} |\alpha(A) - \beta(A)|,
\end{equation}
where $\FM$ is the Borel sigma algebra over $\XM$.

The p-Wasserstein metric between $\alpha$ and $\beta$ is defined as
\begin{equation} \label{eqn_Wasserstein_distance}
	W_p(\alpha, \beta) \defi \min_{\pi\in\Pi} \bigg(\int_{\XM^2} \|x - y\|^p d \pi(x, y) \bigg)^{1/p} ,
\end{equation}
where $\Pi(\alpha, \beta) \defi \{\pi\in\MM_+^1(\XM\times\XM) | \sharp_1\pi = \alpha, \sharp_2\pi = \beta\}$ is the set of joint distributions with given marginal distributions $\alpha$ and $\beta$.
Here $\sharp_i$ denotes the marginalization.

\section{Experiment} \label{appendix_experiment}
The code can be found in the anonymous repo \url{https://anonymous.4open.science/r/279e682d-09e3-43cb-ba2d-f77eeac0d53e/}.
\subsection{Structure of the Weak Learner}
For MNIST, the weak learner is a multilayer perceptron with two hidden layers with size 32 and 32. The activation function is leaky relu with negative slop being $0.01$.

For CIFAR10, the weak learner is a CNN with the same structure as the pytorch tutorial \url{https://pytorch.org/tutorials/beginner/blitz/cifar10_tutorial.html}.
The only difference is that for the fully connected layers, the hidden sizes are changed from 120 and 84 to 32 and 32.

\subsection{MNIST Result}
In Figure \ref{fig_mnist_appendix}, we present the results of MNIST with the y-axis denoting the testing accuracy and x-axis denoting the communication cost.
We use the same step size setting for \ffgd\ and the number of local steps $K$ is set to $2$.
For \fedavg\, the client uses 20\% of the local data per local step, and takes 25 local steps. 
These is the suggested parameter setting in \citep{karimireddy2020scaffold} for \fedavg. The local step size is set to 0.0003.
\begin{figure}[h]
	\begin{tabular}{cc}
		\includegraphics[width=.45\columnwidth]{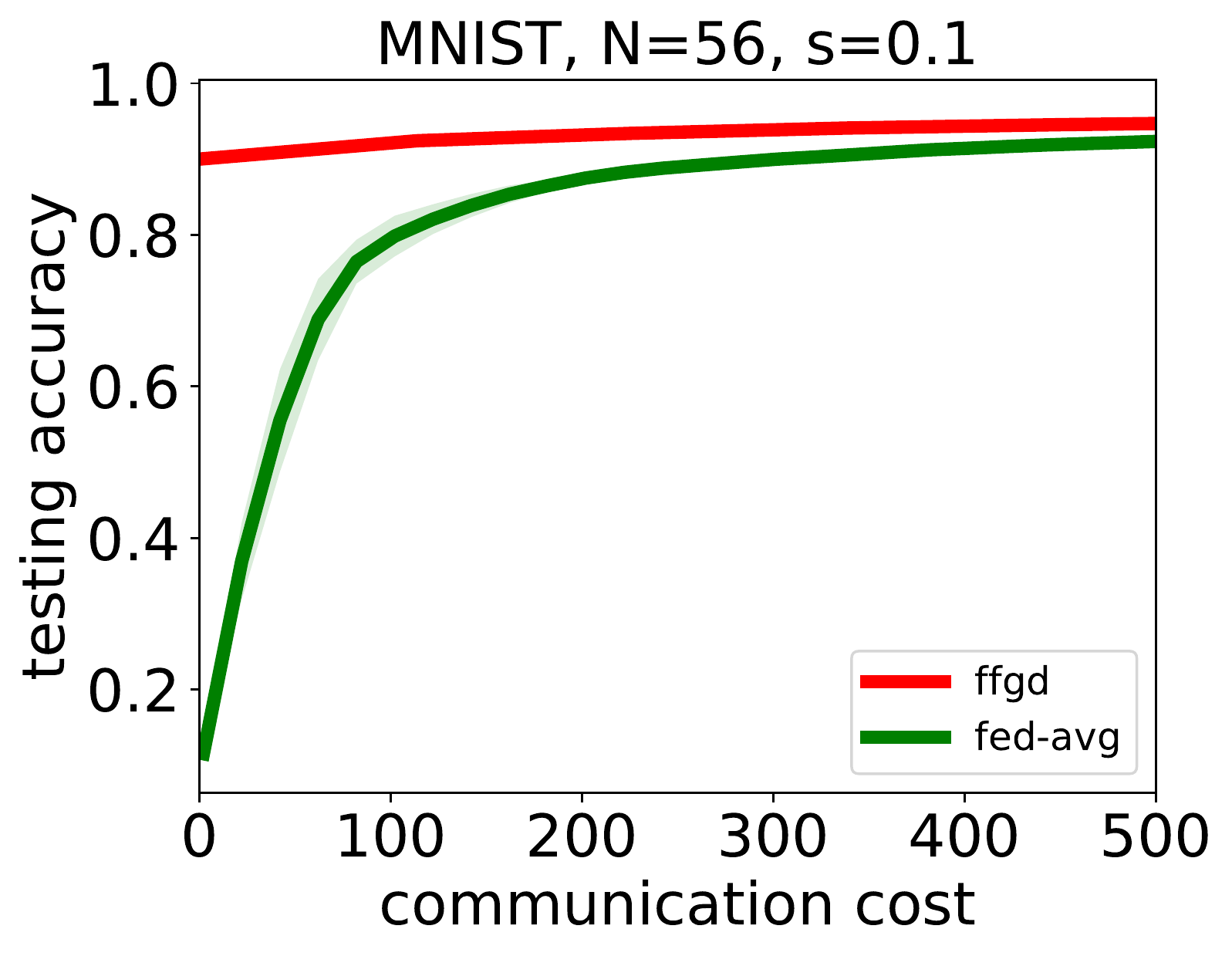} &
		\includegraphics[width=.45\columnwidth]{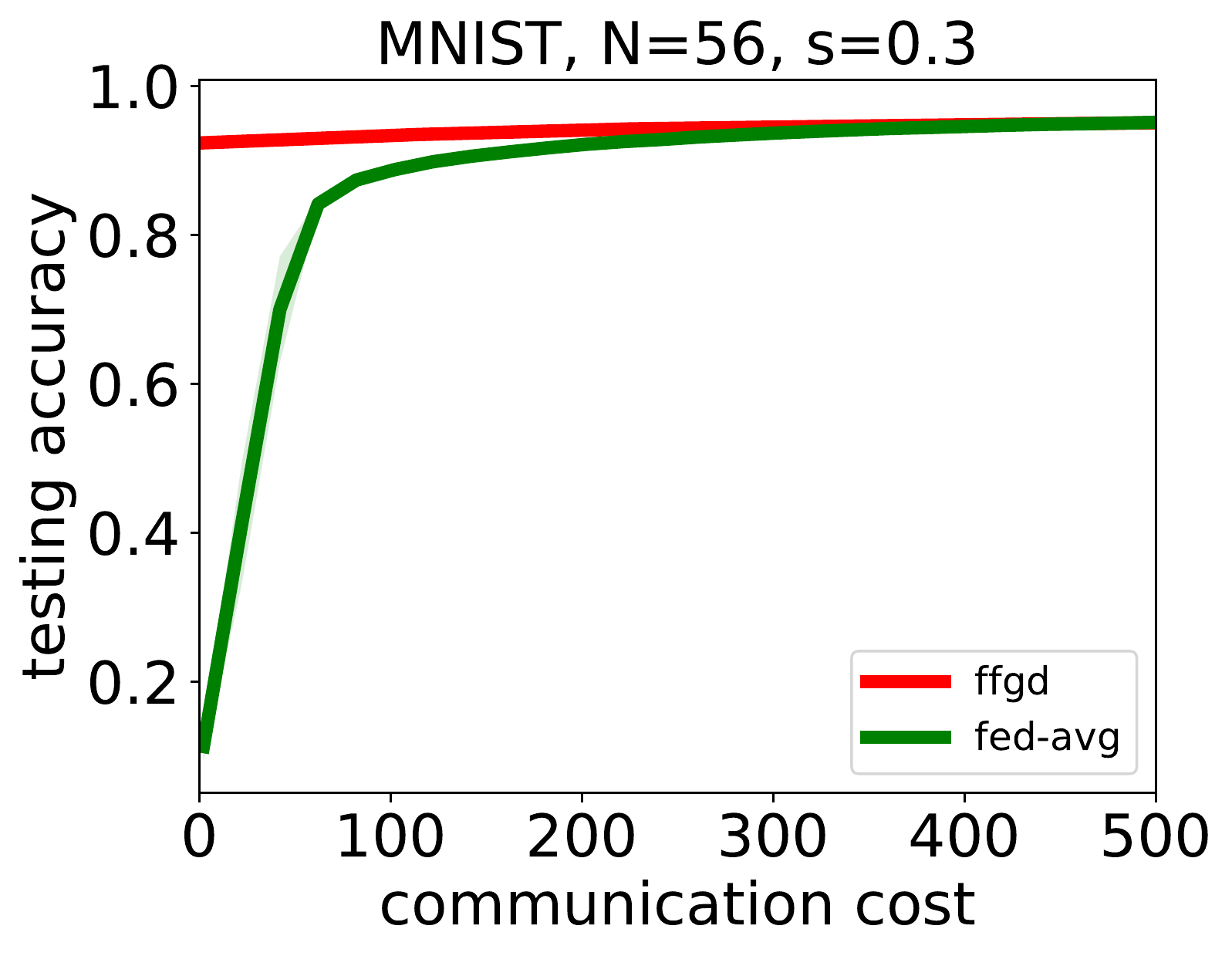}
	\end{tabular}
	\caption{Testing accuracy vs communication cost on MNIST.}
	\label{fig_mnist_appendix}
\end{figure}

\subsection{Implementing the Weak Learning Oracles}
We now discuss the implementations of the weak learning oracles.
In our experiments, we only use \ffgd and hence only the weak learning oracle $\mathrm{WO}^2_\alpha$ is actually implemented.
We discuss the implementation of the oracles $\mathrm{WO}^\infty_\alpha$ and $\mathrm{WO}^\lip_\alpha$ for completeness, but the suggested schemes may not be efficient in practice.

\paragraph{Implementing $\mathrm{WO}^2_\alpha$.}
Let $\phi$ be the input to the oracle and let $h_\theta$ be the candidate weak learner to be trained.
Here $h_\theta$ is a neural network with parameter $\theta$.

We can implement the oracle by solving
\begin{equation}
	\min_\theta \sum_{x\in\supp(\alpha)} \|\phi(x) - h_\theta(x)\|^2.
\end{equation}
In our experiment, we run Adam for $1000$ steps to solve the above problem using initial step size $0.005$.

\paragraph{Implementing $\mathrm{WO}^\infty_\alpha$.}
Let $\phi$ be the input to the oracle and let $h_\theta$ be the candidate weak learner to be trained.
Here $h_\theta$ is a neural network with parameter $\theta$.

We can implement the oracle by solving
\begin{equation}
	\min_\theta  \max_{x\in\supp(\alpha)} \|\phi(x) - h_\theta(x)\|^2.
\end{equation}

\paragraph{Implementing $\mathrm{WO}^\lip_\alpha$.}
Let $\phi$ be the input to the oracle and let $h_\theta$ be the candidate weak learner to be trained.
Here $h_\theta$ is a neural network with parameter $\theta$.

We can implement the oracle by solving
\begin{equation}
	\min_\theta  \left(\max_{x\in\supp(\alpha)} \|\phi(x) - h_\theta(x)\|^2\right) + \int_{\XM} \|\nabla_x \phi(x) - \nabla_x h_\theta(x)\|^2 d x.
\end{equation}
Note that the gradient of the input $\phi$ is available in \ffgdl as we have the explicit expression of $\phi$ in line \ref{eqn_residual_plus_gradient_L} of Algorithm \ref{algorithm_ffgd_L} for all $k=1, \ldots, K$.
The above scheme is similar to the Sobolev training scheme (1) in \citep{czarnecki2017sobolev}.


\section{Proof of Theorem \ref{thm_ffgd}} \label{appendix_proof_theorem_ffgd}
\begin{theorem}[Theorem \ref{thm_ffgd} restated.]
	Let $f^0$ be the global initializer function. Suppose that Assumption \ref{ass_bounded_gradient} holds, and suppose the weak learning oracle $\mathrm{WO}_\alpha^2$ satisfies \eqref{eqn_weak_oracle_contraction} with constant $\gamma$.
	Using the step size $\eta^{k, t} = \frac{2}{\mu(tK+k+1)}$, the output of \ffgd satisfies
	\begin{equation*}
		\begin{aligned}
			\| f^{T} - f^*\|_\alpha^2 = O\Bigg(\frac{\| f^{0} - f^*\|_\alpha^2}{KT} + \frac{KG^2\log (KT)}{T\gamma^2\mu^2} + \frac{(1-\gamma)G^2}{K\mu^2\gamma^2} + \frac{(1-\gamma)G^2\log (KT)}{KT\mu\gamma^2}\Bigg).
		\end{aligned}
	\end{equation*}
\end{theorem}
\begin{proof}

	Since we are considering the semi-heterogeneous case where $\alpha = \alpha_i$, we ignore the subscript $i$ and simply write $\|\cdot\|_{\alpha}$ and $\langle\cdot, \cdot\rangle_{\alpha}$ for the norm and the inner product in $\LM^2(\alpha)$.

	For simplicity, we denote $\hat h_i^{k, t} = h_i^{k, t} + \mu g_i^{k, t}$.
	
	We define two hypothetical global average sequences $\bar g^{k, t} = \frac{1}{N}\sum_{i=1}^{N} g_{i}^{k, t}$ and $\bar h^{k, t} = \frac{1}{N}\sum_{i=1}^{N} \hat h_i^{k, t}$.
	In particular, we have $\bar{g}^{1, t} = f^{t}$.
	From the update rule in line \eqref{eqn_update_variable} of Algorithm \ref{algorithm_ffgd}, we write
	\begin{equation} \label{eqn_expande_update_rule}
		\|\bar g^{k+1, t} - f^*\|_\alpha^2 = \|\bar g^{k, t} - f^*\|_\alpha^2 + (\eta^{k, t})^2\|\bar h^{k, t}\|_\alpha^2 - 2\eta^{k, t}\langle \bar g^{k, t} - f^*, \bar h^{k, t}\rangle_\alpha
	\end{equation}
	The last term of  \eqref{eqn_expande_update_rule} can be split as
	\begin{align}
		-2\langle \bar g^{k, t} - f^*, \bar h^{k, t}\rangle_{\alpha}=&\ -\frac{2}{N}\sum_{i=1}^{N}\langle \bar g^{k, t} - f^*, \hat h_i^{k, t}\rangle_{\alpha}= -\frac{2}{N}\sum_{i=1}^{N}\langle \bar g^{k, t} - g_{i}^{k, t}, \hat h_i^{k, t}\rangle_{\alpha}+ \langle g_{i}^{k, t} - f^*, \hat h_i^{k, t}\rangle_{\alpha}\notag \\
		=&\ \frac{2}{N}\sum_{i=1}^{N}\langle  g_{i}^{k, t} - \bar g^{k, t}, \hat h_i^{k, t}\rangle_{\alpha}
		+ \langle f^* - g_{i}^{k, t}, \nabla \FM_{i}[g_{i}^{k, t}]\rangle_{\alpha}
		+ \langle f^* - g_{i}^{k, t}, \hat h_i^{k, t} - \nabla \FM_{i}[g_{i}^{k, t}]\rangle_\alpha. \label{eqn_split_2}
	\end{align}
	The second term of  \eqref{eqn_split_2} can be bounded using the $\mu$-strong convexity of $\FM_{i}$
	\begin{align*}
		\frac{2}{N}\sum_{i=1}^{N}\langle f^* - g_{i}^{k, t}, \nabla \FM_{i}[g_{i}^{k, t}]\rangle_{\alpha}\leq&\ - \frac{2}{N}\sum_{i=1}^{N}\FM_{i}[g_{i}^{k, t}] - \FM_{i}[f^*] - \frac{2}{N}\sum_{i=1}^{N}\frac{\mu}{2}\|f^* - g_{i}^{k, t}\|^2_\alpha.
	\end{align*}
	Note that using the optimality of $f^*$ we have $$\sum_{i=1}^{N}\FM_{i}[g_{i}^{k, t}] - \FM_{i}[f^*] = \sum_{i=1}^{N}\FM_{i}[g_{i}^{k, t}] - \FM_{i}[\bar g^{k, t}] + \FM_{i}[\bar g^{k, t}] - \FM_{i}[f^*] \leq \sum_{i=1}^{N}\FM_{i}[g_{i}^{k, t}] - \FM_{i}[\bar g^{k, t}]$$
	and that by recalling that $\bar g^{k, t} = \frac{1}{N}\sum_{i=1}^{N} g_{i}^{k, t}$ and using the Cauchy–Schwarz inequality we have
	$$\frac{1}{N}\sum_{i=1}^{N}\|f^* - g_{i}^{k, t}\|_\alpha^2 \geq \|f^* - \bar g^{k, t}\|_\alpha^2.$$
	Therefore, we can bound
	\begin{align*}
		&\ \frac{2}{N}\sum_{i=1}^{N}\langle f^* - g_{i}^{k, t}, \nabla \FM_{i}[g_{i}^{k, t}]\rangle_{\alpha}\\
		\leq &\ - \frac{2}{N}\sum_{i=1}^{N}\FM_{i}[g_{i}^{k, t}] - \FM_{i}[\bar g^{k, t}] - \frac{\mu}{2}\|f^* - \bar g^{k, t}\|^2_\alpha - \frac{1}{N}\sum_{i=1}^{N}\frac{\mu}{2}\|f^* - g_{i}^{k, t}\|^2_\alpha \\
		\leq &\ - \frac{2}{N}\sum_{i=1}^{N}\langle \nabla \FM_{i}[\bar g^{k, t}], g_{i}^{k, t} - \bar g^{k, t}\rangle_{\alpha}- \frac{\mu}{2}\|f^* - \bar g^{k, t}\|^2_\alpha - \frac{1}{N}\sum_{i=1}^{N}\frac{\mu}{2}\|f^* - g_{i}^{k, t}\|^2_\alpha \\
		\leq &\ \frac{1}{N}\sum_{i=1}^{N} \eta^{k, t}\|\nabla \FM_{i}[\bar g^{k, t}]\|^2_\alpha + \frac{1}{\eta^{k, t}}\|g_{i}^{k, t} - \bar g^{k, t}\|^2_\alpha - \frac{\mu}{2}\|f^* - \bar g^{k, t}\|^2_\alpha - \frac{1}{N}\sum_{i=1}^{N}\frac{\mu}{2}\|f^* - g_{i}^{k, t}\|^2_\alpha,
	\end{align*}
	where we use Young's inequality in the last inequality.
	
	Besides, recall that $\|\nabla \RM_i[g]\|_\alpha\leq G$ from Assumption \ref{ass_bounded_gradient}. 
	Together with the property of the oracle, we have
	\begin{equation}
		\|\Delta_i^{k, t}\|_\alpha \leq (1-\gamma) \left(\|\Delta_i^{k-1, t}\|_\alpha + G\right)\ \mathrm{and}\ \|\Delta_i^{0, t}\|_\alpha = 0 \Rightarrow \forall k, \|\Delta_i^{k, t}\|_\alpha\leq \frac{1-\gamma}{\gamma} G.
	\end{equation}
	Consequently, we also have
	\begin{align*}
		\|\bar h^{k}\|_\alpha \leq \frac{1}{N}\sum_{i=1}^{N}\|h_i^{k, t}\|_\alpha = \frac{1}{N}\sum_{i=1}^{N}\|\nabla \RM_{i}[g_{i}^{k, t}] + \Delta_{i}^{k-1}\|_\alpha \leq \frac{2-\gamma}{\gamma}G.
	\end{align*}
	From line \ref{eqn_update_variable} of Algorithm \ref{algorithm_ffgd}, we have $g_i^{k+1, t} = (1-\mu\eta^{k, t}) g_i^{k, t} + \eta^{k, t} h_i^{k, t}$ and therefore
	\begin{equation}
		\|g_i^{k+1, t}\|_{\alpha} - \frac{2G}{\gamma\mu} \leq (1-\mu\eta^{k, t}) \left(\|g_i^{k+1, t}\|_{\alpha} - \frac{2G}{\gamma\mu}\right),
	\end{equation}
	where we use $\|h_i^{k, t}\|_{\alpha} \leq 2G/\gamma$.
	Therefore, if we have initially $\|f^t\|_{\alpha}\leq \frac{2G}{\gamma\mu}$, we always have $\|g_i^{k, t}\|_{\alpha} \leq \frac{2G}{\gamma\mu}$ (hence so is $f^{t+1}$ as it is the global average $\bar g^{K+1, t}$).
	Further, together with $\|h_i^{k, t}\|_{\alpha} \leq 2G/\gamma$, we have $\|\hat h_i^{k, t}\|_{\alpha} \leq 4G/\gamma$.
	
	Additionally, $\frac{1}{N}\sum_{i=1}^{N}\|g_{i}^{k, t} - \bar g^{k, t}\|^2_\alpha$ can be bounded by (we use $E[\left(X-E[X]\right)^2]\leq E[X^2]$)
	\begin{align*}
		\frac{1}{N}\sum_{i=1}^{N}\|g_{i}^{k, t} - \bar g^{k, t}\|^2_\alpha =&\ \frac{1}{N}\sum_{i=1}^{N}\|g_{i}^{k, t} - g_i^{1, t} + g_i^{1, t} - \bar g^{k, t}\|^2_\alpha \\
		\leq &\ \frac{1}{N}\sum_{i=1}^{N}\|g_{i}^{k, t} - g_i^{1, t}\|^2_\alpha \leq \sum_{\kappa=1}^{k}16(\eta^{\kappa, t})^2G^2/\gamma^2 \leq 64(\eta^{k, t})^2K^2G^2/\gamma^2,
	\end{align*}
	where we use $\eta^{\kappa, t} \leq 2\eta^{k, t}$ for any $t\geq0$ and $1\leq \kappa\leq k$.
	Therefore we can bound the first term of  \eqref{eqn_split_2} by
	\begin{align*}
		\frac{2}{N}\sum_{i=1}^{N}\langle  g_{i}^{k, t} - \bar g^{k, t}, \hat h_i^{k, t}\rangle_{\alpha}\leq \frac{1}{N}\sum_{i=1}^{N}\frac{1}{\eta^{k, t}}\|g_{i}^{k, t} - \bar g^{k, t}\|^2_\alpha + \frac{\eta^{k, t}}{N}\sum_{i=1}^{N}\|\hat h_i^{k, t}\|^2_\alpha \leq (64K^2+16)\eta^{k, t}G^2/\gamma^2
	\end{align*}
	Plug in the above results into  \eqref{eqn_expande_update_rule} to yield
	\begin{align*}
		\|\bar g^{k+1, t} - f^*\|_\alpha^2 \leq&\ (1-\frac{\mu\eta^{k, t}}{2})\|\bar g^{k, t} - f^*\|_\alpha^2 + O\left((\eta^{k, t})^2K^2G^2/\gamma^2\right)\\ &\
		+ \frac{2\eta^{k, t}}{N}\sum_{i=1}^N\langle f^* - g_{i}^{k, t}, \hat h_i^{k, t} - \nabla \FM_{i}[g_{i}^{k, t}]\rangle_{\alpha}- \frac{\eta^{k, t}}{N}\sum_{i=1}^{N}\frac{\mu}{2}\|f^* - g_{i}^{k, t}\|^2.
	\end{align*}
	Recall that  $\eta^{k, t} = \frac{2}{\mu(tK + k + 1)}$ and multiply both sides by $(tK + k + 1)$
	\begin{align*}
		(tK + k + 1)\|\bar g^{k+1, t} - f^*\|^2_\alpha \leq&\ (tK + k)\|\bar g^{k, t} - f^*\|^2_\alpha + O(\frac{K^2G^2}{\gamma^2\mu^2(tK + k + 1)}) \\
		&\ + \frac{4}{\mu N}\sum_{i=1}^N\langle f^* - g_{i}^{k, t}, \hat h_i^{k, t} - \nabla \FM_{i}[g_{i}^{k, t}]\rangle_{\alpha}- \frac{2}{\mu N}\sum_{i=1}^{N}\frac{\mu}{2}\|f^* - g_{i}^{k, t}\|^2_\alpha.
	\end{align*}
	Sum from $k=1$ to $K$
	\begin{align}
		(tK + K + 1)\|\bar g^{K+1, t} - f^*\|^2_\alpha \leq&\ (tK+1)\|\bar g^{1, t} - f^*\|^2_\alpha + O(\frac{K^2G^2}{\gamma^2\mu^2}) \left(\log(tK+K) - \log(tK+1)\right) 
		\notag \\
		&\	 + \frac{4}{\mu N}\sum_{i=1}^N\sum_{k=1}^{K}\langle f^* - g_{i}^{k, t}, \hat h_i^{k, t} - \nabla \FM_{i}[g_{i}^{k, t}]\rangle_{\alpha}- \frac{2}{\mu N}\sum_{i=1}^{N}\sum_{k=1}^{K}\frac{\mu}{2}\|f^* - g_{i}^{k, t}\|^2_\alpha. \label{eqn_appendix_i}
	\end{align}
	We now focus on the first term of the second line above
	\begin{align}
		&\ \sum_{k=1}^{K} \langle f^* - g_{i}^{k, t}, \hat h_i^{k, t} - \nabla \FM_{i}[g_{i}^{k, t}]\rangle_\alpha = \sum_{k=1}^{K} \langle f^* - g_{i}^{k, t}, h_i^{k, t} - \nabla \RM_{i}[g_{i}^{k, t}]\rangle_\alpha \notag \\
		=&\ \sum_{k=1}^{K}\langle f^* - g_{i}^{k, t}, h_i^{k, t} - (\nabla \RM_{i}[g_{i}^{k, t}] + \Delta_{i}^{k-1})\rangle_{\alpha}+ \sum_{k=1}^{K}\langle f^* - g_{i}^{k, t}, \Delta_{i}^{k-1}\rangle_{\alpha}\notag \\
		=&\ \sum_{k=1}^{K}\langle f^* - g_{i}^{k, t}, - \Delta_{i}^{k}\rangle_{\alpha}+ \sum_{k=2}^{K}\langle f^* - g_{i}^{k, t}, \Delta_{i}^{k-1}\rangle_{\alpha}+ \langle f^* - g_{i}^{1}, \Delta_{i}^{0}\rangle_{\alpha}&&\& \Delta_{i}^0 = 0\notag \\
		=&\ \sum_{k=1}^{K}\langle f^* - g_{i}^{k, t}, - \Delta_{i}^{k}\rangle_{\alpha}+ \sum_{k=1}^{K-1}\langle f^* - g_{i}^{k+1}, \Delta_{i}^{k}\rangle_{\alpha}\notag \\
		=&\ \sum_{k=1}^{K}\langle f^* - g_{i}^{k, t}, - \Delta_{i}^{k}\rangle_{\alpha}+ \sum_{k=1}^{K-1}\langle f^* - g_{i}^{k, t}, \Delta_{i}^{k}\rangle_{\alpha}+ \sum_{k=1}^{K-1}\langle \eta^{k}_{t} h_i^{k, t}, \Delta_{i}^{k}\rangle_{\alpha}\notag \\
		=&\ \langle f^* - g_{i}^{K, t}, -\Delta_{i}^{K}\rangle_{\alpha}+ \sum_{k=1}^{K-1}\langle \eta^{k}_{t} h_i^{k, t}, \Delta_{i}^{k}\rangle_{\alpha}\notag \\
		\leq &\ \frac{\mu}{2}\|f^* - g_{i}^{K, t}\|^2 + \frac{1}{2\mu}(\frac{1-\gamma}{\gamma})^2G^2 + \frac{(1-\gamma)(2-\gamma)}{\gamma^2}G^2\sum_{k=1}^{K-1} \eta^{k, t}\notag  \\
		\leq&\ \frac{\mu}{2}\|f^* - g_{i}^{K, t}\|^2 + \frac{1}{2\mu}(\frac{1-\gamma}{\gamma})^2G^2 + \frac{(1-\gamma)(2-\gamma)}{\gamma^2}G^2(\log(tK+K) - \log(tK + 2)). \label{eqn_residual_cancel}
	\end{align}
	Using this result, we obtain (we cancel $\frac{\mu}{2}\|f^* - g_{i}^{K, t}\|^2$ with the last term of \eqref{eqn_appendix_i})
	\begin{align*}
		\left(K(t+1) + 1\right)\| f^{t+1} - f^*\|^2_\alpha \leq&\ (Kt+1)\| f^{t} - f^*\|^2_\alpha + O(\frac{K^2G^2}{\gamma^2\mu^2}) (\log(K(t+1) + 1) - \log(Kt + 1)) 
		\\&\	 + \frac{4}{\mu} (\frac{1}{2\mu}(\frac{1-\gamma}{\gamma})^2G^2 + \frac{(1-\gamma)(2-\gamma)}{\gamma^2}G^2(\log(K(t+1) -1 ) - \log(Kt+1)))
	\end{align*}
	Sum from $t = 0$ to $T-1$ to yield
	\begin{align*}
		(KT + 1)\| f^{T} - f^*\|^2_\alpha \leq \| f^{0} - f^*\|^2_\alpha + O(\frac{K^2G^2\log (KT)}{\gamma^2\mu^2}) + O(\frac{(1-\gamma)TG^2}{\mu^2\gamma^2}) + O(\frac{(1-\gamma)G^2\log (KT)}{\mu\gamma^2}),
	\end{align*}
	and hence
	\begin{align*}
		\| f^{T} - f^*\|^2_\alpha \leq O(\frac{\| f^{0} - f^*\|^2_\alpha}{KT}) + O(\frac{KG^2\log (KT)}{T\gamma^2\mu^2}) + O(\frac{(1-\gamma)G^2}{K\mu^2\gamma^2}) + O(\frac{(1-\gamma)G^2\log (KT)}{KT\mu\gamma^2}).
	\end{align*}
\end{proof}
\section{Proof of Theorem \ref{thm_ffgdc}}
We restate \ffgdc in Algorithm \ref{algorithm_ffgdc_appendix} as there is a typo in Algorithm \ref{algorithm_ffgdc}. 
Specifically, in line \ref{eqn_update_variable_c}, we missed the term $\mu g_i^{k, t}$. 
This is marked in red in Algorithm \ref{algorithm_ffgdc_appendix}. 
Our result further requires that $B\geq \frac{2G}{\mu\gamma}$.
Since the Tikhonov regularization ensures the feasibility of $f^t$ (and $g_i^{k, t}$), we hence remove the clipping operation in line \ref{eqn_zero_initialization_c} of Algorithm \ref{algorithm_ffgdc} (see line \ref{eqn_zero_initialization_c_appendix} of Algorithm \ref{algorithm_ffgdc_appendix}).
\begin{algorithm}[t]
	\caption{\textsc{Client} procedure for Federated Functional Gradient Descent with Clipping (\ffgdc)}
	\begin{algorithmic}[1]
		\Procedure {Client}{$i$, $t$, $f$}
		\State $\Delta_{i}^{0, t} = 0$, $g_i^{1, t} = {\color{red}f^t}$ \label{eqn_zero_initialization_c_appendix}
		\For{$k\leftarrow 1$ to $K$}
		\State $h_{i}^{k, t}\ := \mathrm{WO}_{\alpha_i}^\infty(\Delta_{i}^{k-1, t} + \nabla \RM_i[g_{i}^{k, t}])$ \label{eqn_residual_plus_gradient_c_appendix}
		\State $g_{i}^{k+1, t} := g_{i}^{k, t} - {\eta^{k, t}}\cdot \left(\Clip{ G^2_\gamma}\left(h_{i}^{k, t}\right) {\color{red} + \mu g_i^{k, t}}\right)$ \label{eqn_update_variable_c_appendix}
		\State $\Delta_{i}^{k, t} := \Clip{G^1_\gamma}\left(\Delta_{i}^{k-1, t} + \nabla \RM_i[g_{i}^{k, t}] - h_{i}^{k, t}\right)$ \label{eqn_update_residual_c_appendix}
		\EndFor
		\State
		\Return $g_i^{K+1, t}$.
		\EndProcedure
	\end{algorithmic}
	\label{algorithm_ffgdc_appendix}
\end{algorithm}
\begin{theorem}[Theorem \ref{thm_ffgdc} restated]
	Let $f^0$ be the global initializer function. Let $\omega = \frac{1}{N}\sum_{i=1}^{N} \TV(\alpha, \alpha_i)$. Set $G_\gamma^1 = \frac{1-\gamma}{\gamma}\cdot G$ and $ G_\gamma^2 = \frac{2-\gamma}{\gamma}\cdot G$.
	Under Assumption \ref{ass_bounded_gradient_c}, and supposing the weak learning oracle $\mathrm{WO}_\alpha^\infty$ satisfies \eqref{eqn_weak_oracle_contraction_c} with constant $\gamma$ and $B\geq \frac{2G}{\mu\gamma}$, using the step sizes $\eta^{k, t} = \frac{4}{\mu(tK+k+1)}$, the output of \ffgdc satisfies
	\begin{equation*}
		\begin{aligned}
			\|f^T -  f^*\|_\alpha^2 = O\bigg({\frac{\|f^0 - f^*\|_\alpha^2}{KT}} + {\frac{KG^2 log(KT)}{T\mu^2\gamma^2}}
			+ \frac{(1-\gamma)^2G^2}{K\mu^2\gamma^2} + {\frac{GB\omega}{\mu\gamma}}\bigg)
		\end{aligned}
	\end{equation*}
\end{theorem}
\begin{proof}
	For simplicity, in this proof, we define $\hat h_i^{k, t} \defi \Clip{G_\gamma^2}(h_{i}^{k, t}) + \mu g_i^{k, t}$.
	
	For a fixed communication round $t$, we define two hypothetical sequences $\bar g^{k, t} = \frac{1}{N}\sum_{i=1}^{N} g_{i}^{k, t}$
	and $\bar h^{k, t} = \frac{1}{N}\sum_{i=1}^{N} \hat h^{k, t}_i$.
	Note that $\bar{g}^{1, t} = f^{t}$.
	From the update rule in line \ref{eqn_update_variable_c} of Algorithm \ref{algorithm_ffgdc}, we write
	\begin{equation} \label{eqn_expand_update_rule_non_iid}
		\|\bar g^{k+1, t} - f^*\|_{\alpha}^2 = \|\bar g^{k, t} - f^*\|_{\alpha}^2 + (\eta^{k, t})^2\|\bar h^{k, t}\|_{\alpha}^2 - 2\eta^{k, t}\langle \bar g^{k, t} - f^*, \bar h^{k, t}\rangle_{\alpha}.
	\end{equation}
	For the second term, we have $\|\bar h^{k, t}\|_{\alpha}^2 \leq \|\bar h^{k, t}\|_{\infty}^2 \leq 4G^2/\gamma^2 = O(\frac{G^2}{\gamma^2})$ due to the clip operation.
	
	The last term of  \eqref{eqn_expand_update_rule_non_iid} can be split as
	\begin{align}
		-2\langle \bar g^{k, t} - f^*, \bar h^{k, t}\rangle_{\alpha}	
		=&\ -\frac{2}{N}\sum_{i=1}^{N}\langle \bar g^{k, t} - f^*, \hat h_{i}^{k, t}\rangle_{\alpha_i} + \left(\langle \bar g^{k, t} - f^*, \hat h_{i}^{k, t}\rangle_{\alpha}- \langle \bar g^{k, t} - f^*, \hat h_{i}^{k, t}\rangle_{\alpha_i}\right) \notag\\
		=&\ \frac{2}{N}\sum_{i=1}^{N}\langle  g_{i}^{k, t} - \bar g^{k, t}, \hat h^{k, t}\rangle_{\alpha_i} 
		+ \langle f^* - g_{i}^{k, t}, \nabla \FM_{i}[g_{i}^{k, t}]\rangle_{\alpha_i}
		+ \langle f^* - g_{i}^{k, t}, \hat h^{k, t} - \nabla \FM_{i}[g_{i}^{k, t}]\rangle_{\alpha_i} \notag \\
		&\ \qquad + \left(\langle \bar g^{k, t} - f^*, \hat h^{k, t}\rangle_{\alpha_i} - \langle \bar g^{k, t} - f^*, \hat h^{k, t}\rangle_{\alpha}\right).  \label{eqn_split_2_non_iid}
	\end{align}
	We first introduce the following lemmas that characterize the boundedness of $h_i^{k, t}$, $g_i^{k, t}$, and $\hat h_i^{k, t}$.
	\begin{lemma} \label{lemma_boundedness_u_i^{k, t}}
		For a fixed $i\in[N]$, we have the following property.
		\begin{enumerate}
			\item On the support of $\alpha_i$, we have $\Delta_{i}^{k, t} := \Delta_{i}^{k-1, t} + \nabla \RM_i[g_{i}^{k, t}] - h_{i}^{k, t}$, that is, on the support of $\alpha_i$, the clip operation has no effect on the support of $\alpha_i$ in line \ref{eqn_update_residual_c} of Algorithm \ref{algorithm_ffgdc}.
			\item  On the support of $\alpha_i$, we have $\Clip{G_\gamma^2}{\left(h_i^{k, t}\right)} = h_i^{k, t}$, that is, on the support of $\alpha_i$, the clip operation has no effect on the support of $\alpha_i$ in line \ref{eqn_update_variable_c} of Algorithm \ref{algorithm_ffgdc}.
		\end{enumerate}
	\end{lemma}
	\begin{proof}
		Note that $\Delta_i^{0, t} \equiv 0$. Besides, for $x\in \supp(\alpha_i)$, in each iteration $|\Delta_i^{k, t}(x)|$ is first increased at most by $G$ after adding $\nabla \RM_i[g_i^{k, t}]$ and is then reduced by at least $1-\gamma$ after subtracting the weak learner $h_i^{k, t}$. 
		Consequently, $|\Delta_i^{k, t}(x)|\leq \frac{1-\gamma}{\gamma}G$ for $x\in\supp(\alpha_i)$:
		\begin{equation}
			\|\Delta_i^{k, t}\|_{\alpha_i, \infty} \leq (1-\gamma) \left(\|\Delta_i^{k-1, t}\|_{\alpha_i, \infty} + G\right)\ \mathrm{and}\ \|\Delta_i^{0, t}\|_{\infty} = 0 \Rightarrow \forall k, \|\Delta_i^{k, t}\|_{\alpha_i, \infty} \leq \frac{1-\gamma}{\gamma} G = G_\gamma^1.
		\end{equation}
		Therefore, the $\mathrm{Clip}$ operation does not affect the values of $\Delta_i^{k, t}$ on $\supp(\alpha_i)$ as they will never exceed $\frac{1-\gamma}{\gamma} G$.
		Further, $\|\Delta_i^{k-1, t} + \nabla \RM_i[g_i^{k, t }]\|_{\alpha_i, \infty} \leq G + \frac{1-\gamma}{\gamma} G = \frac{G}{\gamma}$ and hence 
		\begin{equation}
			\|h_i^{k, t}\|_{\alpha_i, \infty}\leq \|\Delta_i^{k-1, t} + \nabla \RM_i[g_i^{k, t }] - h_i^{k, t}\|_{\alpha_i, \infty} + \|\Delta_i^{k-1, t} + \nabla \RM_i[g_i^{k, t }]\|_{\alpha_i, \infty} \leq \frac{2-\gamma}{\gamma}G = G_\gamma^2.
		\end{equation}
		Therefore, the $\mathrm{Clip}$ operation does not affect the values of $h_i^{k, t}$ on $\supp(\alpha_i)$ neither as they will never exceed $G_\gamma^2$.
	\end{proof}
	\begin{lemma} \label{lemma_boundedness_of_iterate}
		Assume that the initial function satisfies $\|f^0\|_\infty \leq \frac{2G}{\gamma\mu}$.
		Then for all $1\leq k\leq K$ and $t\geq 0$,
		$\|\bar g^{k, t}\|_{\infty} \leq \frac{2G}{\gamma\mu}$.
	\end{lemma}
	\begin{proof}
		For $t=0$, $\bar g^{1, 0} = \Clip{B}\left(f^0\right) = f^0$ (since $B\geq \frac{2G}{\gamma\mu}$) and hence $\|\bar g^{1, 0}\|_{\infty}\leq \frac{2G}{\gamma\mu}$ due to the initialization.
		Now assume that for $t=\tau$ the statement holds. Therefore $\|f^{\tau+1}\|_\infty\leq \frac{2G}{\gamma\mu}$.
		So for $t=\tau+1$, $\|\bar g^{1, t}\|_\infty\leq \frac{2G}{\gamma\mu}$.
		From the update rule in line \eqref{eqn_update_variable_c_appendix}, we have
		\begin{equation}
			\bar g^{k+1, t} = (1-\mu \eta^{k, t})\bar g^{k, t} + \frac{1}{N}\sum_{i=1}^{N} \eta^{k, t}\Clip{G_\gamma^2}(h_{i}^{k, t})
		\end{equation}
		Recursively, we have $$\|\bar g^{k+1, t}\|_\infty - \frac{2G}{\gamma\mu} \leq (1-\mu \eta^{k, t})\left(\|\bar g^{k, t}\|_\infty - \frac{2G}{\gamma\mu}\right),$$ 
		which leads to the conclusion.
	\end{proof}
	Combing the above two lemmas, we have the boundedness of $\hat h_i^{k, t}$ and $\nabla F_i[g_i^{k, t}]$.
	\begin{lemma}
		$\|\hat h_i^{k, t}\|_\infty \leq \frac{4G}{\gamma}$ and $\|\nabla \FM_i[g_i^{k, t}]\|_\infty\leq \frac{4G}{\gamma}$.
	\end{lemma}

	Using the variational formulation of the TV norm, one has (clearly $\|\bar g^{k, t}\|_{\infty}\leq \frac{2G}{\mu\gamma}\leq B$)
	\begin{equation}
		|\langle \bar g^{k, t} - f^*, h_i^{k, t}\rangle_{\alpha}- \langle \bar g^{k, t} - f^*, h_i^{k, t}\rangle_{\alpha_i}|
		\leq O\left(BG/\gamma\cdot\TV(\alpha, \alpha_i)\right).
	\end{equation}
	Denote $\omega \defi \frac{1}{N} \sum_{i=1}^{N} \TV(\alpha, \alpha_i)$. We hence have 
	\begin{equation}
		\frac{2}{N}\sum_{i=1}^{N}\left(\langle \bar g^{k, t} - f^*, h_i^{k, t}\rangle_{\alpha_i} - \langle \bar g^{k, t} - f^*, h_i^{k, t}\rangle_\alpha\right) \leq \delta \defi O(BG\omega/\gamma).
	\end{equation}
	
	The first term of \eqref{eqn_split_2_non_iid} can be bounded by (note that we simply use $h_i^{k, t}$ since $\Clip{G_\gamma^2}(h_i^{k, t}) = h_i^{k, t}$ on $\supp(\alpha_i)$ due to Lemma \ref{lemma_boundedness_u_i^{k, t}})
	\begin{align*}
		\frac{2}{N}\sum_{i=1}^{N}\langle  g_{i}^{k, t} - \bar g^{k, t}, h_i^{k, t}\rangle_{\alpha_i} \leq&\ \frac{1}{N}\sum_{i=1}^{N}\eta^{k, t} \|h_i^{k, t}\|_{\alpha_i}^2 + \frac{1}{\eta^{k, t}}\|g_i^k - \bar g^k\|^2_{\alpha_i} \\
		\leq&\ \eta^{k, t} 4G^2/\gamma^2 + \frac{1}{N}\sum_{i=1}^{N}\frac{1}{\eta^{k, t}}\|g_i^k - \bar g^k\|^2_{\alpha_i} = O(\frac{\eta^{k, t}G^2}{\gamma^2}) + \frac{1}{\eta^{k, t}}\cdot\frac{1}{N}\sum_{i=1}^{N}\|g_i^k - \bar g^k\|^2_{\alpha_i}.
	\end{align*}
	The second term of \eqref{eqn_split_2_non_iid} can be bounded by using the $\mu$-strong convexity of $\FM_{i}$:
	\begin{equation*}
		\frac{2}{N}\sum_{i=1}^{N}\langle f^* - g_{i}^{k, t}, \nabla \FM_{i}[g_{i}^{k, t}]\rangle_{\alpha_i} 
		\leq - \frac{2}{N}\sum_{i=1}^{N}\left(\FM_{i}[g_{i}^{k, t}] - \FM_{i}[f^*]\right) - \frac{2}{N}\sum_{i=1}^{N}\frac{\mu}{2}\|f^* - g_{i}^{k, t}\|_{\alpha_i}^2.
	\end{equation*}
	For the first term above, using the optimality of $f^*$, we have
	\begin{equation*}
		- \frac{2}{N}\sum_{i=1}^{N}\left(\FM_{i}[g_{i}^{k, t}] - \FM_{i}[f^*]\right) = - \frac{2}{N}\sum_{i=1}^{N}\left(\FM_{i}[g_{i}^{k, t}] - \FM_{i}[\bar g^{k, t}] + \FM_{i}[\bar g^{k, t}] - \FM_{i}[f^*]\right) \leq - \frac{2}{N}\sum_{i=1}^{N}\left(\FM_{i}[g_{i}^{k, t}] - \FM_{i}[\bar g^{k, t}]\right).
	\end{equation*}
	For the second term, we have
	\begin{equation*}
		\|f^* - g_{i}^{k, t}\|_{\alpha_i}^2 \leq 2\|f^* - \bar g^{k, t}\|_{\alpha_i}^2 + 2\|\bar g^{k, t} - g_{i}^{k, t}\|_{\alpha_i}^2.
	\end{equation*}
	Combine the above inequality to yield
	\begin{align}
		&\ \frac{2}{N}\sum_{i=1}^{N}\langle f^* - g_{i}^{k, t}, \nabla \FM_{i}[g_{i}^{k, t}]\rangle_{\alpha_i} \notag \\
		\leq &\ - \frac{2}{N}\sum_{i=1}^{N}\left(\FM_{i}[g_{i}^{k, t}] - \FM_{i}[\bar g^{k, t}]\right) - \frac{\mu}{4}\|f^* - \bar g^{k, t}\|_{\alpha}^2 + \frac{\mu}{2N}\sum_{i=1}^{N}\|g_{i}^{k, t} - \bar g^{k, t}\|_{\alpha_i} ^2  - \frac{\mu}{2N}\sum_{i=1}^{N}\|f^* - g_{i}^{k, t}\|_{\alpha_i} ^2\notag\\
		\leq &\ - \frac{2}{N}\sum_{i=1}^{N}\langle \nabla \FM_{i}[\bar g^{k, t}], g_{i}^{k, t} - \bar g^{k, t}\rangle_{\alpha_i} - \frac{\mu}{4}\|f^* - \bar g^{k, t}\|_{\alpha}^2 + \frac{\mu}{2N}\sum_{i=1}^{N}\|g_{i}^{k, t} - \bar g^{k, t}\|_{\alpha_i} ^2 - \frac{\mu}{2N}\sum_{i=1}^{N}\|f^* - g_{i}^{k, t}\|_{\alpha_i} ^2\notag\\
		\leq &\ \frac{1}{N}\sum_{i=1}^{N} \eta^{k, t}\|\nabla \FM_{i}[\bar g^{k, t}]\|_{\alpha_i} ^2 + \frac{1}{\eta^{k, t}}\|g_{i}^{k, t} - \bar g^{k, t}\|_{\alpha_i} ^2 - \frac{\mu}{4}\|f^* - \bar g^{k, t}\|_{\alpha}^2 + \frac{\mu}{2N}\sum_{i=1}^{N}\|g_{i}^{k, t} - \bar g^{k, t}\|_{\alpha_i} ^2 - \frac{\mu}{2N}\sum_{i=1}^{N}\|f^* - g_{i}^{k, t}\|_{\alpha_i} ^2\notag\\
		\leq &\ \eta^{k, t}\frac{16G^2}{\gamma^2} + (\frac{\mu}{2} + \frac{1}{\eta^{k, t}})\frac{1}{N}\sum_{i=1}^{N} \|g_{i}^{k, t} - \bar g^{k, t}\|_{\alpha_i} ^2 - \frac{\mu}{4}\|f^* - \bar g^{k, t}\|_{\alpha}^2 - \frac{\mu}{2N}\sum_{i=1}^{N}\|f^* - g_{i}^{k, t}\|_{\alpha_i} ^2 \label{eqn_analysis_i}
	\end{align}

	Note that $\|\bar g^{k, t} - f^t\|_{\alpha_i}^2 = \|\sum_{\kappa = 1}^{k-1} \eta^{\kappa, t}\bar{h}^{\kappa, t}\|^2_{\alpha_i}$ and $\eta^{t, \kappa}\leq 2\eta^{t, k}$ for $\kappa\leq k$.
	Therefore, $\frac{1}{N}\sum_{i=1}^{N}\|g_{i}^{k, t} - \bar g^{k, t}\|_{\alpha_i}^2$ can be bounded by (we use $\|\cdot\|_{\alpha_i}\leq\|\cdot\|_\infty$ in the following)
	\begin{align*}
		\frac{1}{N}\sum_{i=1}^{N}\|g_{i}^{k, t} - \bar g^{k, t}\|_{\alpha_i}^2 =&\ \frac{1}{N}\sum_{i=1}^{N}\|g_{i}^{k, t} - f^{t} + f^{t} - \bar g^{k, t}\|_{\alpha_i}^2 \\
		\leq &\ \frac{1}{N}\sum_{i=1}^{N}2\|g_{i}^{k, t} - f^{t}\|_{\alpha_i}^2 + 2\|f^{t} - \bar g^{k, t}\|_{\alpha_i}^2 \leq 36\sum_{\kappa=1}^{k-1}(\eta^{\kappa, t})^2G^2/\gamma^2 \leq 144(\eta^{k, t})^2K^2G^2/\gamma^2 \\
		=&\ O(\frac{(\eta^{k, t})^2K^2G^2}{\gamma^2}).
	\end{align*}

	Plug in the above results into  \eqref{eqn_expand_update_rule_non_iid} to yield (note that $\hat h_i^{k, t} - \nabla \FM_{i}[g_{i}^{k, t}] = h_i^{k, t} - \nabla \RM_{i}[g_{i}^{k, t}]$)
	\begin{align*}
		\|\bar g^{k+1, t} - f^*\|_\alpha^2 \leq & (1-\frac{\mu\eta^{k, t}}{4})\|\bar g^{k, t} - f^*\|_\alpha^2  
		+ O(\frac{(\eta^{k, t})^2K^2G^2}{\gamma^2}) + \frac{2\eta^{k, t}}{N}\sum_{i=1}^N\langle f^* - g_{i}^{k, t}, h_i^{k, t} - \nabla \RM_{i}[g_{i}^{k, t}]\rangle_{\alpha_i}  \\
		& + \eta^{k, t} \delta  - \frac{\mu}{2N}\sum_{i=1}^{N}\|f^* - g_{i}^{k, t}\|_{\alpha_i}^2
	\end{align*}
	Recall that  $\eta^{k,  t} = \frac{4}{\mu(tK + k + 1)}$ and multiply both sides by $(Kt + k + 1)$
	\begin{align*}
		(Kt + k + 1)\|\bar g^{k+1, t} - f^*\|^2 \leq&\ (Kt + k)\|\bar g^{k, t} - f^*\|^2 + O(\frac{K^2G^2}{\mu^2\gamma^2(Kt + k + 1)}) \\
		&\ + \frac{4}{\mu N}\sum_{i=1}^N\langle f^* - g_{i}^{k, t}, h_i^{k, t} - \nabla \RM_{i}[g_{i}^{k, t}]\rangle_{\alpha_i} + \frac{2\delta}{\mu}  - \frac{\mu}{2N}\sum_{i=1}^{N}\|f^* - g_{i}^{k, t}\|_{i}^2
	\end{align*}
	Sum from $k=1$ to $K$
	\begin{align*}
		(Kt + K+1)\|\bar g^{k+1, t} - f^*\|^2 \leq&\ (Kt+1)\|\bar g^{1, t} - f^*\|^2 + O\left(\frac{K^2G^2}{\mu^2\gamma^2} (\log(Kt+K+1) - \log(Kt+1)) \right)
		\\&\	 + \frac{4}{\mu N}\sum_{i=1}^N\sum_{k=1}^{K}\langle f^* - g_{i}^{k, t}, h_i^{k, t} - \nabla \RM_{i}[g_{i}^{k, t}]\rangle_{\alpha_i} + \frac{2\delta K}{\mu}  - \frac{2}{N}\sum_{k=1}^{K}\sum_{i=1}^{N}\|f^* - g_{i}^{k, t}\|_{\alpha_i} ^2
	\end{align*}
	For the first on the second line above, the following equality holds for the same reason as \eqref{eqn_residual_cancel} (note that the equality holds due to Lemma \ref{lemma_boundedness_u_i^{k, t}} and the fact that the inner product only depends on values on $\supp(\alpha_i)$)
	\begin{align*}
		 \sum_{k=1}^{K} \langle f^* - g_{i}^{k, t}, h_i^{k, t} - \nabla \RM_{i}[g_{i}^{k, t}]\rangle_{\alpha_i}
		= \frac{\mu}{2}\|f^* - g_{i}^{k, t}\|_{\alpha_i}^2 +  O( G^2\frac{(1-\gamma)^2}{\mu\gamma^2}) + O\left(G^2\frac{1-\gamma}{\gamma^2}(\log(KT+K) - \log(Kt))\right).
	\end{align*}
	Using this result, we obtain
	\begin{align*}
		(K(t+1) + 1)\| f^{t+1} - f^*\|_\alpha^2 \leq&\ (Kt+1)\| f^{t} - f^*\|_\alpha^2 + O\left(\frac{K^2G^2}{\mu^2\gamma^2} (\log(K(t+1)+1) - \log(Kt+1)) \right)
		\\&\	 + O( G^2\frac{(1-\gamma)^2}{\mu^2\gamma^2}) + O\left(G^2\frac{1-\gamma}{\mu\gamma^2}(\log(K(t+1)) - \log(Kt))\right) + \frac{2\delta K}{\mu}
	\end{align*}
	Sum from $t = 0$ to $T-1$ and use the non-expensiveness of the clip operation to yield
	\begin{align*}
		(KT + 1)\| f^{T} - f^*\|_\alpha^2 \leq&\ (k_0+1)\| f^{0} - f^*\|_\alpha^2 + O\left(\frac{K^2G^2\log(KT + 1)}{\mu^2\gamma^2} \right) + O( \frac{(1-\gamma)G^2\log(TK)}{\mu\gamma^2})
		\\&\	+ O( TG^2\frac{(1-\gamma)^2}{\mu^2\gamma^2}) + O(\frac{GTBK\omega}{\mu\gamma}),
	\end{align*}
	and hence
	\begin{equation}
		\|f^T - f^*\|_\alpha^2 = O({\frac{\|f^0 - f^*\|_\alpha^2}{KT}} + {\frac{KG^2 log(KT)}{T\mu^2\gamma^2}} + \frac{(1-\gamma)^2G^2}{K\mu^2\gamma^2} + {\frac{GB\omega}{\mu\gamma}}).
	\end{equation}
\end{proof}
\section{Proof of Theorem \ref{thm_ffgdl}}
We first prove \eqref{eqn_W1_variational_formulation}.

Recall \eqref{eqn_Wasserstein_distance} where the Wasserstein-1 distance between two discrete distribution $\mu$ and $\nu$ can be written as
\begin{align*}
	W_1(\mu, \nu) = \min_{\Pi\geq 0} \int_{\XM^2} \|x-y\|d\Pi(x, y), \quad s.t. \sharp_1 \Pi = \mu, \sharp_2 \Pi = \nu.
\end{align*}
Note that the constraint of the above problem implies that $\supp(\Pi)\subseteq \supp(\mu)\times\supp(\nu)$, otherwise $\Pi$ must be infeasible.
The above minimization problem is equivalent to 
\begin{align*}
	W_1(\mu, \nu) = \min_{\Pi\geq 0} \max_{\phi, \psi} \int_{\XM^2} \|x-y\|d\Pi(x, y) + \int_{\XM} \phi(x) d\left(\mu - \sharp_1 \Pi\right)(x) + \int_{\XM} \psi(y) d\left(\nu - \sharp_2 \Pi\right)(y).
\end{align*}
Change the order of min-max to max-min (due to convexity) and rearrange terms:
\begin{align*}
	W_1(\mu, \nu) = \max_{\phi, \psi} \left\{\min_{\Pi\geq 0}  \int_{\XM^2} \|x-y\| - \phi(x) - \psi(y) d\Pi(x, y)\right\} + \int_{\XM} \phi(x) d \mu(x) + \int_{\XM} \psi(y) d \nu(y).
\end{align*}
Therefore, we must have that for $(x, y) \in \supp(\Pi)\subseteq \supp(\mu)\times\supp(\nu)$, $\phi(x) + \psi(y) \leq \|x- y\|$ which leads to 
\begin{align*}
	W_1(\mu, \nu) = \max_{\phi, \psi} \int_{\XM} \phi(x) d \mu(x) + \int_{\XM} \psi(y) d \nu(y), \quad s.t.\ \phi(x) + \psi(y) \leq \|x- y\|, \forall (x, y)\in \supp(\mu)\times\supp(\nu).
\end{align*}

Now, recall that every local distribution $\alpha_i$ is described by a set of data feature points: $\alpha_i = \frac{1}{M}\sum_{j=1}^{M} \delta_{x_{i,j}}$, where $\delta_x$ is the Dirac distribution;
and the global distribution $\alpha$ is described by the union of all these points: $\alpha = \frac{1}{MN}\sum_{i,j=1}^{N,M} \delta_{x_{i,j}}$.
Clearly we have $\supp(\alpha_i)\subseteq\supp(\alpha)$. 
Using Proposition 6.1. of \citep{peyre2019computational} with $\XM = \supp(\alpha)$, the above bi-variable problem is equivalent to the single-variable problem
\begin{align*}
	W_1(\mu, \nu) = \max_{\phi} \int_{\XM} \phi(x) d \mu(x) - \int_{\XM} \phi(y) d \nu(y), \quad s.t.\ \|\phi(x) - \phi(y)\| \leq \|x- y\|, \forall (x, y)\in \supp(\mu)\times\supp(\nu).
\end{align*}

Recall that in \eqref{eqn_W1_variational_formulation}, $\phi(x) = f(x)g(x)$ and $\xi = \|g\|_{\lip}\|f\|_\infty + \|f\|_{\lip}\|g\|_\infty$.
One can check that
\begin{equation*}
	\|\phi(x)/\xi - \phi(y)/\xi\| = \|f(x)\left(g(x) - g(y)\right) + g(y)\left(f(x) - f(y) \right)\|/\xi \leq \left(\|g\|_{\lip}\|f\|_\infty + \|f\|_{\lip}\|g\|_\infty\right)/\xi\|x-y\| = \|x-y\|.
\end{equation*}
Therefore 
\begin{align*}
	W_1(\mu, \nu) \geq |\int_{\XM} f(x)g(x)/\xi d \alpha (x) - \int_{\XM} f(x)g(x)/\xi d \alpha_i(y)|,
\end{align*}
which is just \eqref{eqn_W1_variational_formulation}.

\begin{theorem}[Theorem \ref{thm_ffgdl} restated]
	Let $f^0$ be the global initializer function. Let $\omega = \frac{1}{N}\sum_{i=1}^{N}\mathrm{W}_1(\alpha, \alpha_i)$ and $G^2 = \frac{2L^2}{N^2}\sum_{i, s=1}^{N} \mathrm{W}_2^2(\alpha_s, \alpha_i) + 2B^2$.
	Consider the federated functional least square minimization problem \eqref{eqn_fed_functional_minimization_L}.
	Under Assumption \ref{ass_bounded_value_h}, and supposing the weak learning oracle $\mathrm{WO}_\alpha^{\lip}$ satisfies \eqref{eqn_weak_oracle_contraction_l1} and \eqref{eqn_weak_oracle_contraction_l2} with constant $\gamma$, for a certain choice of step sizes $\eta^{k,t}$, the output of \ffgdl satisfies
	\begin{equation}
		\|f^T - f^*\|^2 = O\left({\frac{\|f^0 - f^*\|^2}{KT}} + {\frac{K\left(L(LD+B)\omega+ G^2+B^2\right) log(KT)}{T\mu^2\gamma^2}} + {\frac{(1-\gamma)^2B^2}{\mu^2\gamma^2 K}}+ \frac{L(LD+B)\omega}{\gamma^2\mu}\right).
	\end{equation}
\end{theorem}
\begin{proof}
	For a fixed communication round $t$, we define a hypothetical sequence $\bar g^{k, t} = \frac{1}{N}\sum_{i=1}^{N} g_{i}^{k, t}$.
	We also define $\bar h^{k} = \frac{1}{N}\sum_{i=1}^{N} h_{i}^{k}$.
	Note that $\bar{g}^{1, t} = f^{t}$.

	
	From the construction \eqref{eqn_lipschitz_extension}, we have $\|u_i\|_\lip\leq L$. Additionally, with property \eqref{eqn_weak_oracle_contraction_l2} of the oracle, the residual is inductively proved to be $\frac{(1-\gamma)}{\gamma}L$-Lipschitz continuous as follows. For the base case, note that $\|\Delta_i^{0}\|_\lip \equiv 0$. Now, assume that for some $k \geq 1$, we have $\|\Delta_i^{k-1}\|_\lip \leq \frac{(1-\gamma)}{\gamma}L$. Then
	\begin{align*}
		\|\Delta_i^{k}\|_\lip\! =\! \|  h_{i}^{k} - (u_i - \Delta_{i}^{k-1})\|_\lip\! \leq (1-\gamma)\|u_i - \Delta_{i}^{k-1}\|_\lip
		\leq (1-\gamma)(\|\Delta_i^{k-1}\|_\lip + L)
		\leq \tfrac{(1-\gamma)}{\gamma}L.
	\end{align*}
	Therefore, the query to the weak learning oracle is also Lipschitz continuous: $\|\Delta_{i}^{k-1} - u_i\|_\lip\leq L/{\gamma}$, and so is the output, $\|h_i^{k}\|_\lip\leq L/{\gamma}$. Now, the update rule of $g_i^{k, t}$ (line \ref{eqn_update_variable_L}), 
	and the boundedness of $\|h_i^{k}\|_\lip$ imply the boundedness of $\|g_i^{k, t}\|_\lip$ for sufficiently small $\eta^{k, t}$:
	\begin{equation*}
		\begin{aligned}
			\|g_i^{k+1, t}\|_\lip =&\ \|(1-\eta^{k, t}) g_i^{k, t} +  \eta^{k, t} h_i^{k}\|_\lip  
			\leq (1-\eta^{k, t}) \|g_i^{k, t}\|_\lip + {L}/{\gamma}\cdot\eta^{k, t}  \\ 
			\Rightarrow&\ \|g_i^{k, t}\|_\lip \leq {L}/{\gamma} \ (\text{via induction using} \|g_i^{1, t}\|_\lip\leq {L}/{\gamma}).
		\end{aligned}
	\end{equation*}

	\begin{lemma} \label{lemma_appendix_h}
		The residual $\Delta_i^k$ and the output $h_i^k$ of the oracle $\mathrm{WO}_{\alpha_i}^\lip$ are bounded under the $\LM^{\infty}(\alpha_i)$ norm: $$\|\Delta_i^k\|_{\alpha_i, \infty} \leq \frac{(1-\gamma) B}{\gamma} \text{ and } \|h_i^k\|_{\alpha_i, \infty} \leq {B}/{\gamma}.$$
	\end{lemma}
	\begin{proof}
		From property \eqref{eqn_weak_oracle_contraction_l1} of the weak leaner oracle $\mathrm{WO}_{\alpha_i}^\lip$, we have
		\begin{equation}
			\|\Delta_i^{k}\|_{\alpha_i, \infty} = \|\Delta_i^{k-1} - u_i + h_i^k\|_{\alpha_i, \infty} \leq (1-\gamma)\|\Delta_i^{k-1} - u_i\|_{\alpha_i, \infty} \leq (1-\gamma)\|\Delta_i^{k-1}\|_{\alpha_i, \infty} + (1-\gamma) B,
		\end{equation}
		where the second inequality uses the boundedness of $y_{i, j} = f_i^*(x_{i,j})$ in Assumption \ref{ass_bounded_value_h}.
		We hence have
		\begin{equation}
			\|\Delta_i^{k}\|_{\alpha_i, \infty} - \frac{(1-\gamma) B}{\gamma} \leq (1-\gamma)\left(\|\Delta_i^{k-1}\|_{\alpha_i, , \infty}  - \frac{(1-\gamma) B}{\gamma} \right) \Rightarrow \|\Delta_i^{k}\|_{\alpha_i, \infty} \leq \frac{(1-\gamma) B}{\gamma}.
		\end{equation}
		The boundedness of $\|h_i^k\|_{\alpha_i, \infty}$ can be obtained from the above inequality: $\|h_i^k\|_{\alpha_i, \infty} \leq \|u_i\|_{\alpha_i, \infty} + \|\Delta_i^{k}\|_{\alpha_i, \infty} \leq {B}/{\gamma}$.
	\end{proof}
	\begin{lemma} \label{lemma_appendix_g}
		The local variable function $g_i^{k, t}$ is bounded under the $\LM^{\infty}(\alpha_i)$ norm: $\|g_i^{k, t}\|_{\alpha_i, \infty} \leq B/\gamma$.
	\end{lemma}
	\begin{proof}
	Using the update rule in line \ref{eqn_update_variable_L} of Algorithm \ref{algorithm_ffgd_L}, we have
	\begin{equation*}
		\|g_i^{k+1, t}\|_{\alpha_i, \infty} = \|(1-\eta^{k, t}) g_i^{k, t} + \eta^{k, t} h_i^{k}\|_{\alpha_i, \infty} \leq (1-\eta^{k, t}) \|g_i^{k, t}\|_{\alpha_i, \infty} + \eta^{k, t} \|h_i^{k}\|_{\alpha_i, \infty} \leq (1-\eta^{k, t}) \|g_i^{k, t}\|_{\alpha_i, \infty} + \eta^{k, t} B/\gamma.
	\end{equation*}
	Inductively, we have the boundedness of $\|g_i^{k+1, t}\|_{\alpha_i, \infty}$
	\begin{equation}
		\|g_i^{k+1, t}\|_{\alpha_i, \infty} - B/\gamma \leq (1-\eta^{k, t}) \left(\|g_i^{k, t}\|_{\alpha_i, \infty} - B/\gamma\right) \Rightarrow \|g_i^{k, t}\|_{\alpha_i, \infty} \leq B/\gamma.
	\end{equation}
	\end{proof}
	\begin{lemma}
		The local variable function $g_i^{k, t}$, the global average function $\bar g^{k, t}$, and the output of the oracle $\mathrm{WO}_{\alpha_i}^\lip$ are $(LD+B)/\gamma$-bounded under the $\LM^{\infty}(\alpha)$ norm.
	\end{lemma}
	\begin{proof}
		From Lemmas \ref{lemma_appendix_g} and \ref{lemma_appendix_h}, $g_i^{k, t}$, $\bar g^{k, t}$ and $h_i^{k, t}$ are $B/\gamma$ on the support of $\alpha_i$.
		Using Assumption \ref{ass_bounded_value_h} together with the $L/\gamma$-Lipschitz continuity of $g_i^{k, t}$, $\bar g^{k, t}$ and $h_i^{k, t}$, we have the results.
	\end{proof}
	While the above lemma implies the boundedness of $\bar g^{k, t}$ under the $\LM^2(\alpha)$ norm, we can tighten the analysis with the following lemma.
	Important, the following result does not depend on the constant $D$ in Assumption \ref{ass_bounded_value_h}.
	\begin{lemma}
		The hypothetical global sequences $\bar g^{k, t}$ and $\bar h^{k, t}$ are bounded under the local norm $\|\cdot\|_{\alpha_s}$:
		Denote $G_s^2 = \frac{2L^2}{N}\sum_{i=1}^{N} \mathrm{W}_2^2(\alpha_s, \alpha_i) + 2B^2$. We have that $\|\bar g^{k, t}\|_{\alpha_s} \leq G_s^2/\gamma^2$ and $\|\bar h^{k, t}\|_{\alpha_s} \leq G_s^2/\gamma^2$, where $\mathrm{W}_2(\alpha_s, \alpha_i)$ is the Wasserstein-2 distance between measures $\alpha_i$ and $\alpha_s$.
		Consequently, we have $\bar g^{k, t}$ and $\bar h^{k, t}$ are $G$-bounded under the $\LM^2(\alpha)$ norm, where we further denote $G^2 = \frac{1}{N}\sum_{s=1}^{N} G_s^2$.
	\end{lemma}
	\begin{proof}
	Let $\Pi^{s, i} \in \RBB_+^{M\times M}$ be the Wasserstein-2 optimal transport plan (matrix) between $\alpha_s$ and $\alpha_i$. The entry $\Pi_{j_1, j_2}^{s, i}$ denotes the portion of mass that should be transported from $x_{s, j_1}\in\supp(\alpha_s)$ to $x_{i, j_2}\in\supp(\alpha_i)$.
	Note that in $\alpha_i$ and $\alpha_s$, the entries $x_{s, j_1}$ and $x_{i, j_2}$ have uniform weight $1/M$.
	As a transport plan, any row or column of $\Pi^{s, i}$ sums up to $1/M$.
	We now show that $\|\bar g^{k, t}\|_{\alpha_{s}}$ is bounded using the Lipschitz continuity of $g^{k, t}_i$.
	\begin{equation} \label{eqn_norm_of_average_function}
		\|\bar g^{k, t}\|^2_{\alpha_{s}} = \frac{1}{M} \sum_{j=1}^{M}\|\frac{1}{N}\sum_{i=1}^{N} g^{k, t}_i(x_{s, j})\|^2 \leq \frac{1}{M} \sum_{j=1}^{M}\frac{1}{N}\sum_{i=1}^{N}\| g^{k, t}_i(x_{s, j})\|^2 = \frac{1}{N}\sum_{i=1}^{N} \frac{1}{M} \sum_{j=1}^{M}\| g^{k, t}_i(x_{s, j})\|^2.
	\end{equation}
	We analyze the summand as follows.
	\begin{align*}
		\frac{1}{M} \sum_{j=1}^{M}\| g^{k, t}_i(x_{s, j})\|^2 =&\ \sum_{j_1=1}^{M} \sum_{j_2=1}^{M} \Pi^{s, i}_{j_1, j_2}\| g^{k, t}_i(x_{s, j_1}) - g^{k, t}_i(x_{i, j_2}) + g^{k, t}_i(x_{i, j_2}) \|^2 \\
		\leq&\  \sum_{j_1=1}^{M} \sum_{j_2=1}^{M} \Pi^{s, i}_{j_1, j_2} \left( 2 \| g^{k, t}_i(x_{s, j_1}) - g^{k, t}_i(x_{i, j_2})\|^2 + 2 \|g^{k, t}_i(x_{i, j_2}) \|^2\right) \\
		\leq&\  \sum_{j_1=1}^{M} \sum_{j_2=1}^{M} \Pi^{s, i}_{j_1, j_2} \left( 2 L^2/\gamma^2\cdot \|x_{s, j_1} - x_{i, j_2}\|^2 + 2\|g^{k, t}_i(x_{i, j_2}) \|^2 \right) \\
		=&\ 2L^2/\gamma^2\cdot \mathrm{W}_2^2(\alpha_s, \alpha_i) + \frac{2}{M} \sum_{j_2=1}^{M} \|g^{k, t}_i(x_{i, j_2}) \|^2 \\
		=&\ 2L^2/\gamma^2\cdot \mathrm{W}_2^2(\alpha_s, \alpha_i) + 2 \|g^{k, t}_i\|_{\alpha_i}^2,
	\end{align*}
	where we used the definition of the Wasserstein-2 distance.
	Therefore, \eqref{eqn_norm_of_average_function} can be bounded by
	\begin{equation}
		\|\bar g^{k, t}\|^2_{\alpha_{s}} \leq \frac{1}{N}\sum_{i=1}^{N} 2L^2/\gamma^2\cdot \mathrm{W}_2^2(\alpha_s, \alpha_i) + 2 \|g^{k, t}_i\|_{\alpha_i}^2 \leq \frac{2L^2}{N\gamma^2}\sum_{i=1}^{N} \mathrm{W}_2^2(\alpha_s, \alpha_i) + 2B^2/\gamma^2.
	\end{equation}
	Following the similar proof above, we have the same bound for $\|\bar h^{k, t}\|_{\alpha_{s}}$ as $h_i^{k, t}$ is also $B/\gamma$- bounded and $L/\gamma$-Lipschitz continuous:
	\begin{equation}
		\|\bar h^{k, t}\|^2_{\alpha_{s}} \leq \frac{2L^2}{N\gamma^2}\sum_{i=1}^{N} \mathrm{W}_2^2(\alpha_s, \alpha_i) + 2B^2/\gamma^2.
	\end{equation}
\end{proof}
	
	We now present the convergence analysis of Algorithm \ref{algorithm_ffgd_L}.
	From the update rule in line \ref{eqn_update_variable_L} of Algorithm \ref{algorithm_ffgd_L}, we write
	\begin{equation} \label{eqn_expand_update_rule_non_iid_l}
		\|\bar g^{k+1, t} - f^*\|_{\alpha}^2 = \|\bar g^{k, t} - f^*\|_{\alpha}^2 + (\eta^{k, t})^2\|\bar g^{k, t} - \bar h^{k}\|_{\alpha}^2 - 2\eta^{k, t}\langle \bar g^{k, t} - f^*, \bar g^{k, t} - \bar h^{k}\rangle_{\alpha}.
	\end{equation}
	To bound the second term, note that 
	\begin{equation}
		\|\bar g^{k, t} - \bar h^{k}\|_{\alpha}^2 \leq \frac{1}{N}\sum_{i=1}^{N} \|g_i^{k, t} - h_{i}^{k}\|_{\alpha}^2.
	\end{equation}
	For each individual term on the R.H.S. of the above inequality, we have
	\begin{equation}
		\|g_i^{k, t} - h_{i}^{k}\|_{\alpha}^2 = \left(\|g_i^{k, t} - h_{i}^{k}\|_{\alpha}^2 - \|g_i^{k, t} - h_{i}^{k}\|_{\alpha_i}^2\right) + \|g_i^{k, t} - h_{i}^{k}\|_{\alpha_i}^2 \leq O(L(LD+B)/\gamma^2\cdot\mathrm{W}_1(\alpha, \alpha_i)) + O(B^2/\gamma^2),
	\end{equation} 
	where we use the variational formulation \eqref{eqn_W1_variational_formulation} of the Wasserstein-1 distance as well as the Lipschitz continuity and boundedness of $g_i^{k, t}$ and $h_i^{k, t}$ under the $\LM^2(\alpha)$ norm.
	Therefore the second term is bounded by
	\begin{equation}
		\|\bar g^{k+1, t} - f^*\|_{\alpha}^2 \leq O(L(LD+B)/\gamma^2\cdot\omega) + O(B^2/\gamma^2), \quad \omega = \frac{1}{N}\sum_{i=1}^{N} \mathrm{W}_1(\alpha, \alpha_i).
	\end{equation}
	
	The third term of \eqref{eqn_expand_update_rule_non_iid_l} can be split as
	\begin{align}
		&\ - 2\langle \bar g^{k, t} - f^*, \bar g^{k, t} - \bar h^{k}\rangle_{\alpha}	\\
		=&\ -\frac{2}{N}\sum_{i=1}^{N}\langle \bar g^{k, t} - f^*, g_i^{k, t} -  h_i^{k}\rangle_{\alpha_i} + \left(\langle \bar g^{k, t} - f^*, g_i^{k, t} -  h_i^{k}\rangle_{\alpha}- \langle \bar g^{k, t} - f^*, g_i^{k, t} -  h_i^{k}\rangle_{\alpha_i}\right) \notag\\
		=&\ \frac{2}{N}\sum_{i=1}^{N}\langle  g_{i}^{k, t} - \bar g^{k, t}, g_i^{k, t} -  h^{k}_i\rangle_{\alpha_i} 
		+ \langle f^* - g_{i}^{k, t}, g_{i}^{k, t} - u_i\rangle_{\alpha_i}
		+ \langle f^* - g_{i}^{k, t}, u_i - h_{i}^{k, t}\rangle_{\alpha_i} \notag \\
		&\ \qquad + \left(\langle \bar g^{k, t} - f^*, g_i^{k, t} -  h_i^{k}\rangle_{\alpha}- \langle \bar g^{k, t} - f^*, g_i^{k, t} -  h_i^{k}\rangle_{\alpha_i}\right).  \label{eqn_split_2_non_iid_l}
	\end{align}
	The last term of R.H.S. of the above equality can be bounded using the Lipschitz continuity and the boundedness of $(\bar g^{k, t} - f^*)$ and $(g_i^{k, t} -  h_i^{k})$ and the variational formulation of $\mathrm{W}_1$ (see \eqref{eqn_W1_variational_formulation}):
	\begin{equation}
		\frac{1}{N}\sum_{i=1}^{N}\left(\langle \bar g^{k, t} - f^*, g_i^{k, t} -  h_i^{k}\rangle_{\alpha}- \langle \bar g^{k, t} - f^*, g_i^{k, t} -  h_i^{k}\rangle_{\alpha_i}\right) = O(L(LD+B)/\gamma^2\cdot\omega), \quad \omega = \frac{1}{N}\sum_{i=1}^{N} \mathrm{W}_1(\alpha, \alpha_i).
	\end{equation}
	The first term of of the R.H.S. of \eqref{eqn_split_2_non_iid_l} can be bounded by
	\begin{align*}
		&\ \frac{2}{N}\sum_{i=1}^{N}\langle  g_{i}^{k, t} - \bar g^{k, t}, g_i^{k, t} -  h_i^{k}\rangle_{\alpha_i} \leq \frac{1}{N}\sum_{i=1}^{N}\eta^{k, t} \|g_i^{k, t} -  h_i^{k}\|_{\alpha_i}^2 + \frac{1}{\eta^{k, t}}\|g_i^{k, t} - \bar g^{k, t}\|^2_{\alpha_i} \\
		\leq&\ O(\eta^{k, t} L(LD+B)/\gamma^2\cdot\mathrm{W}_1(\alpha, \alpha_i)) + O( \eta^{k, t} B^2/\gamma^2) + \frac{1}{\eta^{k, t}}\cdot\frac{1}{N}\sum_{i=1}^{N}\|g_i^{k, t} - \bar g^k\|^2_{\alpha_i}.
	\end{align*}

	The second term of \eqref{eqn_split_2_non_iid_l} can be bounded by using the $\mu$-strong convexity of $\RM_{i}$ (note that $\mu=1$ and we use $\nabla \RM_{i}[g_{i}^{k, t}]$ to denote $(g_i^{k, t} - u_i)$ as they are identical on the support of $\alpha_i$).
	The following inequality holds for the same reason as \eqref{eqn_analysis_i}.
	\begin{align*}
		&\ \frac{2}{N}\sum_{i=1}^{N}\langle f^* - g_{i}^{k, t}, \nabla \RM_{i}[g_{i}^{k, t}]\rangle_{\alpha_i} \\
		\leq &\ O\left(\eta^{k, t}G^2/\gamma^2\right) + (\frac{\mu}{2} + \frac{1}{\eta^{k, t}})\frac{1}{N}\sum_{i=1}^{N} \|g_{i}^{k, t} - \bar g^{k, t}\|_{\alpha_i} ^2 - \frac{\mu}{4}\|f^* - \bar g^{k, t}\|^2 - \frac{\mu}{2N}\sum_{i=1}^{N}\|f^* - g_{i}^{k, t}\|_{\alpha_i} ^2
	\end{align*}


Note that $\|f^t - \bar g^{k, t}\|_{\alpha_i}^2 = \|\sum_{\kappa = 1}^k \eta^{\kappa, t}\left(\bar g^{\kappa, t} - \bar{h}^{\kappa}\right)\|^2_{\alpha_i}$ and $\eta^{t, \kappa}\leq 2\eta^{t, k}$ for $\kappa\leq k$.
Therefore, $\frac{1}{N}\sum_{i=1}^{N}\|g_{i}^{k, t} - \bar g^{k, t}\|_{\alpha_i}^2$ can be bounded by
\begin{align*}
	\frac{1}{N}\sum_{i=1}^{N}\|g_{i}^{k, t} - \bar g^{k, t}\|_{\alpha_i}^2 =&\ \frac{1}{N}\sum_{i=1}^{N}\|g_{i}^{k, t} - f^{t} + f^{t} - \bar g^{k, t}\|_{\alpha_i}^2 \\
	\leq &\ \frac{1}{N}\sum_{i=1}^{N}2\|g_{i}^{k, t} - f^{t}\|_{\alpha_i}^2 + 2\|f^{t} - \bar g^{k, t}\|_{\alpha_i}^2 = O(\sum_{\kappa=1}^{k}(\eta^{\kappa, t})^2G^2/\gamma^2) = O(\eta^{k, t})^2K^2G^2/\gamma^2 \\
	=&\ O(\frac{(\eta^{k, t})^2K^2G^2}{\gamma^2}).
\end{align*}

Plug in the above results into  \eqref{eqn_expand_update_rule_non_iid_l} to yield 
\begin{align*}
	\|\bar g^{k+1, t} - f^*\|^2_\alpha \leq &\ (1-\frac{\mu\eta^{k, t}}{2})\|\bar g^{k, t} - f^*\|^2_\alpha  
	+ O\left(\frac{(\eta^{k, t})^2K^2}{\gamma^2}\left(L(LD+B)\omega+ G^2+B^2\right)\right) + \frac{2\eta^{k, t}}{N}\sum_{i=1}^N\langle f^* - g_{i}^{k, t}, u_i - h_i^{k}\rangle_{\alpha_i} \\
	&\ + O(\eta^{k, t}L(LD+B)/\gamma^2\cdot\omega) - \frac{\mu\eta^{k, t}}{2N}\sum_{i=1}^{N}\|f^* - g_{i}^{k, t}\|_{\alpha_i}^2
\end{align*}
Set  $\eta^{k,  t} = \frac{4}{\mu(Kt + k + 1)}$ and multiply both sides by $(Kt + k + 1)$
\begin{align*}
	(Kt + k + 1)\|\bar g^{k+1, t} - f^*\|^2 \leq&\ (Kt + k)\|\bar g^{k, t} - f^*\|^2 + O(\frac{K^2\left(L(LD+B)\omega+ G^2+B^2\right)}{\mu^2\gamma^2(Kt + k + 1)}) \\
	&\ + \frac{4}{\mu N}\sum_{i=1}^N\langle f^* - g_{i}^{k, t}, u_i - h_i^{k}\rangle_{\alpha_i} +
	O\left((L(LD+B)\cdot\omega/(\gamma^2\mu)\right) - \frac{1}{N}\sum_{i=1}^{N}\|f^* - g_{i}^{k, t}\|_{i}^2
\end{align*}
Sum from $k=1$ to $K$
\begin{align*}
	(Kt + K+1)\|\bar g^{k+1, t} - f^*\|^2 \leq&\ (Kt+1)\|\bar g^{1, t} - f^*\|^2 + O\left(\frac{K^2\left(L(LD+B)\omega+ G^2+B^2\right)}{\mu^2\gamma^2} (\log(Kt+K+1) - \log(Kt+1)) \right)
	\\&\	 + \frac{4}{\mu N}\sum_{i=1}^N\sum_{k=1}^{K}\langle f^* - g_{i}^{k, t}, u_i - h_i^{k}\rangle_{\alpha_i} + O\left(KL(LD+B)\cdot\omega/(\gamma^2\mu)\right) - \frac{1}{N}\sum_{k=1}^K\sum_{i=1}^{N}\|f^* - g_{i}^{k, t}\|_{\alpha_i} ^2
\end{align*}
We now focus on the last term (note that the equality holds due to Lemma \ref{lemma_boundedness_u_i^{k, t}} and the fact that the inner product only depends on values on $\supp(\alpha_i)$)
\begin{align*}
	&\ \sum_{k=1}^{K} \langle f^* - g_{i}^{k, t}, u_i - h_i^{k}\rangle_{\alpha_i}\\
	=&\ \sum_{k=1}^{K}\langle f^* - g_{i}^{k, t}, u_i - (u_i - \Delta_{i}^{k-1} + \Delta_i^k)\rangle_{\alpha_i} = \sum_{k=1}^{K}\langle f^* - g_{i}^{k, t}, \Delta_{i}^{k-1} - \Delta_i^k\rangle_{\alpha_i}\\
	=&\ \sum_{k=1}^{K}\langle f^* - g_{i}^{k, t}, - \Delta_{i}^{k}\rangle_{\alpha_i} + \sum_{k=2}^{K}\langle f^* - g_{i}^{k, t}, \Delta_{i}^{k-1}\rangle_{\alpha_i} + \langle f^* - g_{i}^{1}, \Delta_{i}^{0}\rangle_{\alpha_i} &&\& \Delta_{i}^0 = 0\\
	=&\ \sum_{k=1}^{K}\langle f^* - g_{i}^{k, t}, - \Delta_{i}^{k}\rangle_{\alpha_i} + \sum_{k=1}^{K-1}\langle f^* - g_{i}^{k+1}, \Delta_{i}^{k}\rangle_{\alpha_i} \\
	=&\ \sum_{k=1}^{K}\langle f^* - g_{i}^{k, t}, - \Delta_{i}^{k}\rangle_{\alpha_i} + \sum_{k=1}^{K-1}\langle f^* - g_{i}^{k, t}, \Delta_{i}^{k}\rangle_{\alpha_i} + \sum_{k=1}^{K-1}\langle \eta^{k}_{t} \left(g_i^{k, t} - h_i^k\right), \Delta_{i}^{k}\rangle_{\alpha_i} \\
	=&\ \langle f^* - g_{i}^{K, t}, -\Delta_{i}^{K}\rangle_{\alpha_i} + \sum_{k=1}^{K-1}\langle \eta^{k}_{t} \left(g_i^{k, t} - h_i^k\right), \Delta_{i}^{k}\rangle_{\alpha_i} \\
	\leq &\ \mu\|f^* - g_i^{k, t}\|_{\alpha_i}^2 + O((1-\gamma)^2B^2/(\gamma^2\mu)) + O(G^2\frac{1-\gamma}{\gamma^2})\sum_{k=1}^{K-1} \eta^{k, t}\\
	 =&\ \mu\|f^* - g_i^{k, t}\|_{\alpha_i}^2 +  O((1-\gamma)^2B^2/(\gamma^2\mu)) + O\left(G^2\frac{1-\gamma}{\gamma^2}(\log(KT+K) - \log(Kt))\right).
\end{align*}
Using this result, we obtain
\begin{align*}
	&\ (K(t+1) + 1)\| f^{t+1} - f^*\|^2 \\
	\leq&\ (Kt+1)\| f^{t} - f^*\|^2 + O\left(\frac{K^2\left(L(LD+B)\omega+ G^2+B^2\right)}{\mu^2\gamma^2} (\log(K(t+1)+1) - \log(Kt+1)) \right)
	\\&\	 + \frac{2}{\mu} (O((1-\gamma)^2B^2/(\gamma^2\mu)) + O\left(G^2\frac{1-\gamma}{\gamma^2}(\log(K(t+1)) - \log(Kt))\right)) + O\left(KL(LD+B)\cdot\omega/(\gamma^2\mu)\right)
\end{align*}
Sum from $t = 0$ to $T-1$ and use the non-expensiveness of the projection operation (note that $\CM$ is a convex set) to yield
\begin{align*}
	(KT + 1)\| f^{T} - f^*\|^2 \leq&\ (k_0+1)\| f^{0} - f^*\|^2 + O\left(\frac{K^2\left(L(LD+B)\omega+ G^2+B^2\right)\log(KT + 1)}{\mu^2\gamma^2} \right) + O(\frac{(1-\gamma)^2TB^2}{\mu^2\gamma^2} \\
	&\ + \frac{(1-\gamma)G^2\log(TK)}{\mu\gamma^2}) + O\left(TKL^2\cdot\omega/(\gamma^2\mu)\right),
\end{align*}
and hence
\begin{equation}
	\|f^T - f^*\|^2 = O({\frac{\|f^0 - f^*\|^2}{KT}} + {\frac{K\left(L(LD+B)\omega+ G^2+B^2\right) log(KT)}{T\mu^2\gamma^2}} + {\frac{(1-\gamma)^2B^2}{\mu^2\gamma^2 K}}+ \frac{L(LD+B)\omega}{\gamma^2\mu}).
\end{equation}
\end{proof}

\section{Partial Device Participation}
\label{app_client_sampling}
In this section, we consider the setting of partial device participation.
In the following discussion, we take Algorithm \ref{algorithm_ffgd} for example. 
Similar arguments hold for Algorithms \ref{algorithm_ffgdc} and \ref{algorithm_ffgd_L}.

In round $t$ of Algorithm \ref{algorithm_ffgd}, we randomly sample without replacement a subset $\SM_t \subseteq [N]$ of clients and only compute the average of their returns to update the global variable function.
We assume that all $\SM_t$ has the same cardinality $m$.
Conceptually, we can imagine all the clients are participating in the update, but we only utilize the results in the set $\SM_t$.

Similar to the proof for the setting of full device participation, we define the hypothetical global average function $\bar g^{k, t} = \frac{1}{N}\sum_{i=1}^{N} g_{i}^{k, t}$.
In particular, we have $\bar g^{1, t} = f^t$.
From the derivation therein (see Section \ref{appendix_proof_theorem_ffgd}), we have
\begin{align*}
	(tK + K + 1)\|\bar g^{K+1, t} - f^*\|^2_\alpha \leq&\ (tK + 1)\|\bar g^{1, t} - f^*\|^2_\alpha + O(\frac{K^2G^2}{\gamma^2\mu^2}) (\log(tK + K + 1) - \log(tK + +1)) 
	\\&\	 + \frac{4}{\mu} (\frac{1}{2\mu}(\frac{1-\gamma}{\gamma})^2G^2 + \frac{(1-\gamma)(2-\gamma)}{\gamma^2}G^2(\log(tK + K) - \log(tK)))
\end{align*}
However, unlike the setting of full device participation, we do not have $f^{t+1} = \bar g^{K+1, t}$.
Instead, $f^{t+1} = \frac{1}{m}\sum_{i\in\SM^t} g_i^{K+1}$.
We have the following simple but useful lemma. The proof of this lemma is similar to scheme II of Lemma 5 in \citep{li2019convergence}.
\begin{lemma} \label{lemma_sampling_without_replacement}
	$\EBB_{\SM_t} \| f^{t+1} - \bar g^{K+1, t}\|_\alpha^2 = O\left(\frac{N-m}{N-1}\frac{(\eta^{K, t})^2 K^2G^2}{m\gamma^2}\right)$.
\end{lemma}
Moreover, $f^{t+1}$ is an unbiased estimator of $\bar g^{K+1, t}$. Therefore $\EBB_{\SM^t} \langle f^{t+1} - \bar g^{K+1, t}, \bar g^{K+1, t} - f^*\rangle_\alpha = 0$ and
\begin{equation}
	\EBB_{\SM^t}\|f^{t+1} - f^*\|^2_\alpha = \EBB_{\SM^t}\|f^{t+1} - \bar g^{K+1, t}\|^2 + \|\bar g^{K+1, t} - f^*\|^2_\alpha.
\end{equation}
Recall that  $\eta^{k, t} = \frac{2}{\mu(tK + k + 1)}$.
Combining the above results, we have
\begin{align*}
	\left((t+1)K+1\right)\EBB_{\SM_t}\|f^{t+1} - f^*\|^2_\alpha \leq&\ (tK + 1)\|f^t - f^*\|^2_\alpha + O(\frac{K^2G^2}{\gamma^2\mu^2}) (\log(tK + K + 1) - \log(tK +1)) \\
	&\	 + \frac{4}{\mu} (\frac{1}{2\mu}(\frac{1-\gamma}{\gamma})^2G^2 + \frac{(1-\gamma)(2-\gamma)}{\gamma^2}G^2(\log(tK + K) - \log(tK))) \\
	&\ + O\left(\frac{N-m}{N-1}\frac{\eta^{K, t} K^2G^2}{\mu m\gamma^2}\right).
\end{align*}
Sum the above results from $t=0$ to $T$, we have
\begin{align*}
	\left((T+1)K+1\right)\EBB\|f^{T+1} - f^*\|^2_\alpha \leq&\ \|f^0 - f^*\|^2_\alpha + O(\frac{K^2G^2\log(KT)}{\gamma^2\mu^2}) + O(\frac{TG^2}{\mu^2}(\frac{1-\gamma}{\gamma})^2) \\
	&\ + O\left(\frac{(1-\gamma)}{\mu\gamma^2}G^2\log(TK + K)\right) + O(\frac{N-m}{N-1}\frac{KG^2\log T}{\mu^2 m\gamma^2}) \\
	\Rightarrow \EBB\|f^{T+1} - f^*\|^2_\alpha = O\bigg( \frac{\|f^0 - f^*\|^2_\alpha}{KT} +&\ \frac{KG^2\log(KT)}{T\gamma^2\mu^2} + \frac{G^2}{K\mu^2}(\frac{1-\gamma}{\gamma})^2 + \frac{(1-\gamma)G^2\log (TK)}{KT\mu\gamma^2} + \frac{N-m}{N-1}\frac{G^2\log T}{T\mu^2 m\gamma^2} \bigg).
\end{align*}

\begin{theorem}
	Let $f^0$ be the global initializer function.
	Suppose that Assumption \ref{ass_bounded_gradient} holds, and suppose that the weak learning oracle $\mathrm{WO}_\alpha^2$ satisfies \eqref{eqn_weak_oracle_contraction} with constant $\gamma$.
	We pick the step size $\eta^{k, t} = \frac{2}{\mu(tK+k+1)}$ and in each round the server randomly selects a subset $\SM_t\subseteq[N]$ without replacement with $|\SM_t|=m$.
	The output of \ffgd\ (Algorithm \ref{algorithm_ffgd}) satisfies
	\begin{equation*}
		\| f^{T} - f^*\|_\alpha^2 = O\left(\frac{\| f^{0} - f^*\|_\alpha^2}{KT} + \frac{KG^2\log (KT)}{T\gamma^2\mu^2} + \frac{(1-\gamma)G^2}{K\mu^2\gamma^2} + \frac{(1-\gamma)G^2\log (KT)}{KT\mu\gamma^2} + \frac{N-m}{N-1}\frac{G^2\log T}{T\mu^2 m\gamma^2}\right).
	\end{equation*}
\end{theorem}

\begin{theorem} 
	Let $f^0$ be the global initializer function. Let $\omega = \frac{1}{N}\sum_{i=1}^{N} \TV(\alpha, \alpha_i)$. Set $G_\gamma^1 = \frac{1-\gamma}{\gamma}\cdot G$ and $ G_\gamma^2 = \frac{2-\gamma}{\gamma}\cdot G$.
	We pick the step size $\eta^{k, t} = \frac{4}{\mu(tK+k+1)}$ and in each round the server randomly selects a subset $\SM_t\subseteq[N]$ without replacement with $|\SM_t|=m$.
	Under Assumption \ref{ass_bounded_gradient_c}, and supposing the weak learning oracle $\mathrm{WO}_\alpha^\infty$ satisfies \eqref{eqn_weak_oracle_contraction_c} with constant $\gamma$, the output of \ffgdc satisfies
	\begin{equation*}
		\begin{aligned}
			\|f^T - f^*\|_\alpha^2 = O\bigg({\frac{\|f^0 - f^*\|_\alpha^2}{KT}} + {\frac{KG^2 log(KT)}{T\mu^2\gamma^2}} 
			 + \frac{(1-\gamma)^2G^2}{K\mu^2\gamma^2} + {\frac{GB\omega}{\mu\gamma}} + {\frac{G^2\log(TK)\omega}{\gamma^2\mu^2 T}} + \frac{N-m}{N-1}\frac{G^2\log T}{T\mu^2 m\gamma^2}\bigg).
		\end{aligned}
	\end{equation*}
\end{theorem}

\begin{theorem}
	Let $f^0$ be the global initializer function. Let $\omega = \frac{1}{N}\sum_{i=1}^{N}\mathrm{W}_1(\alpha, \alpha_i)$ and $G^2 = \frac{2L^2}{N^2}\sum_{i, s=1}^{N} \mathrm{W}_2^2(\alpha_s, \alpha_i) + 2B^2$.
	Consider the federated functional least square minimization problem \eqref{eqn_fed_functional_minimization_L}.
	We pick the step size $\eta^{k, t} = \frac{4}{\mu(tK+k+1)}$ and in each round the server randomly selects a subset $\SM_t\subseteq[N]$ without replacement with $|\SM_t|=m$.
	Under Assumption \ref{ass_bounded_value_h}, and supposing the weak learning oracle $\mathrm{WO}_\alpha^\infty$ satisfies \eqref{eqn_weak_oracle_contraction_l1} and \eqref{eqn_weak_oracle_contraction_l2} with constant $\gamma$, the output of \ffgdl satisfies
	\begin{equation*}
		\begin{aligned}
			\|f^T - f^*\|_\alpha^2 =  O\Bigg({\frac{\|f^0 - f^*\|^2}{KT}}&\ + {\frac{K\left(L(LD+B)\omega+ G^2+B^2\right) log(KT)}{T\mu^2\gamma^2}} \\
			&\ + {\frac{(1-\gamma)^2B^2}{\mu^2\gamma^2 K}}+ \frac{L(LD+B)\omega}{\gamma^2\mu} + \frac{N-m}{N-1}\frac{G^2\log T}{T\mu^2 m\gamma^2}\Bigg).
		\end{aligned}
	\end{equation*}
\end{theorem}